%% file: stream_clipper_kdd18_arxiv.tex
\documentclass[sigconf]{acmart}

\usepackage[utf8]{inputenc} 
\usepackage{hyperref}       
\usepackage{url}            
\usepackage{booktabs}       
\usepackage{amsfonts}       
\usepackage{nicefrac}       
\usepackage{microtype}      
\usepackage{times}
\usepackage{epsfig}
\usepackage{wrapfig}
\usepackage{float}
\usepackage{amsmath}
\usepackage{amssymb}
\usepackage{amsthm}
\usepackage{rotating}
\usepackage[linesnumbered,ruled,vlined]{algorithm2e}
\usepackage{multirow}
\usepackage{multicol}
\usepackage{blindtext}
\newtheorem{lemma}{Lemma}
\newtheorem{theorem}{Theorem}

\newtheorem{proposition}{Proposition}
\newtheorem{corollary}{Corollary}

\renewcommand{\theenumii}{\theenumi.\arabic{enumii}.}
\DeclareMathOperator*{\argmax}{argmax}

\setcopyright{none}

\begin{document}
\title{Stream Clipper: \\Scalable Submodular Maximization on Stream}
\author{Tianyi Zhou}
\affiliation{%
  \institution{University of Washington, Seattle}
}
\email{tianyizh@uw.edu}

\author{Jeff Bilmes}
\affiliation{%
  \institution{University of Washington, Seattle}
}
\email{bilmes@uw.edu}

\begin{abstract}
	We propose a streaming submodular maximization algorithm ``stream clipper'' that performs as well as the offline greedy algorithm on document/video summarization in practice. It adds elements from a stream either to a solution set $S$ or to an extra buffer $B$ based on two adaptive thresholds, and improves $S$ by a final greedy step that starts from $S$ adding elements from $B$. During this process, swapping elements out of $S$ can occur if doing so yields improvements. The thresholds adapt based on if current memory utilization exceeds a budget, e.g., it increases the lower threshold, and removes from the buffer $B$ elements below the new lower threshold. We show that, while our approximation factor in the worst case is $1/2$ (like in previous work, and corresponding to the tight bound), we show that there are data-dependent conditions where our bound falls within the range $[1/2, 1-1/e]$. In news and video summarization experiments, the algorithm consistently outperforms other streaming methods, and, while using significantly less computation and memory, performs similarly to the offline greedy algorithm. 
\end{abstract} 


\maketitle

\input{stream_clipper_body_arxiv}

\end{document}

%% file: stream_clipper_body_arxiv.tex
\section{Introduction}

Success in today's machine learning and artificial intelligence algorithms relies largely on big data. Often, however, there may exist a small data subset that can act as a surrogate for the whole. Thus, various summarization methods have been designed to select such representative subsets and reduce redundancy. The problem is usually formulated as maximizing a score function $f(S)$ that assigns importance scores to subsets $S\subseteq V$ of an underlying ground set $V$ of all elements. Submodular functions are a useful class of functions for this purpose: a function $f:2^V\rightarrow \mathbb R$ is submodular~\cite{Fujishige} if for any subset $A\subseteq B\subseteq V$ and $v \notin B$, 
\begin{equation}
f(v\cup A) - f(A)\geq f(v\cup B) - f(B).
\end{equation}
Since the above diminishing returns property naturally captures the redundancy among elements in terms of their importance to a summary, submodular functions have been commonly used as objectives in summarization and machine learning applications. The importance of $v$'s contribution to $A$ is $f(v|A)\triangleq f(v\cup A) - f(A)$, called the ``marginal gain'' of $v$ conditioned on $A$.

The objective $f(\cdot)$ can be chosen from a rich class of submodular functions, e.g., facility location, saturated coverage, feature based, entropy and $\log\det(\cdot)$. We focus on the most commonly used form: normalized and monotone non-decreasing submodular functions, i.e., $f(v|A) > 0, \forall v\in V\backslash A, A\subseteq V$ and $f(\emptyset) = 0$. In order for a summary $S$ to have a limited size, a cardinality constraint is often applied, as we focus on in this paper. We also address, however, knapsack and matroid constraints in \cite{Supp}. Under a cardinality constraint, the problem becomes
\begin{equation}\label{equ:card}
\max_{S\subseteq V,\\|S|\leq k}f(S).
\end{equation}
Submodular maximization is usually NP-hard. However, (\ref{equ:card}) can be solved near-optimally by a greedy algorithm with approximation factor $1-1/e$ \cite{greedyapprox}. Starting from $S\leftarrow\emptyset$, greedy algorithm selects the element with the largest marginal gain $f(v|S)$ into $S$, i.e., $S\leftarrow S\cup\argmax_{v\in V\backslash S}f(v|S)$, until $|S|=k$. To accelerate the greedy algorithm without an objective value loss, the lazy greedy approach~\cite{lazygreedy, LeskovecLazyGreedy} updates only the top element of a priority queue of marginal gains for all elements in $V\backslash S$ in each step. Recent approximate greedy algorithms~\cite{fast-submodular-semigradient, fast-multi-stage, lazier-than-lazy} develop piece-wise, multi-stage, or random sampling strategies to tradeoff approximate optimality and speed.

\begin{figure*}
	\begin{center}
		\includegraphics[width=1\linewidth]{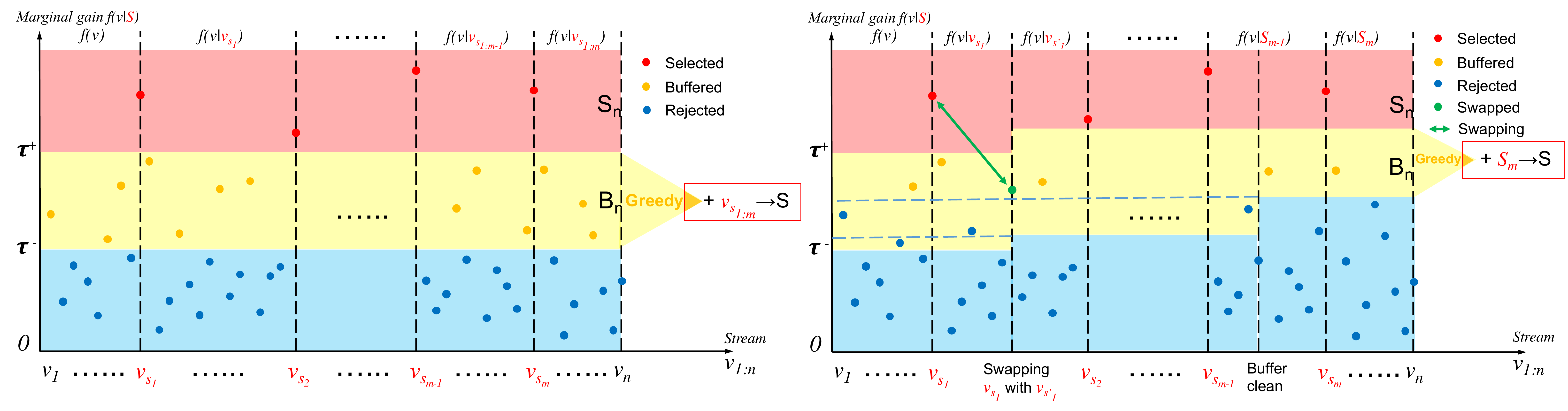}
	\end{center}\vspace{-3mm}
	\caption{Left: Na\"{i}ve stream clipper in Algorithm~\ref{alg:StreamBuffer}, $(v_1, v_2, \cdots, v_n)$ is the stream, $(v_{s_1}, v_{s_2}, \cdots, v_{s_m})$ is the sequence of the $m$ selected elements; Right: Stream clipper with swapping and buffer cleaning in Algorithm~\ref{alg:StreamClipper}, swapping replaces $v_{s_1}$ with $v_{s'_1}$ and increases both $\tau^-$ and $\tau^+$, buffer cleaning removes elements by increasing $\tau^-$.}
	\label{fig:sc}
\end{figure*}

In various applications such as news digesting, video summarization \cite{MirzasoleimanJ017}, music recommending and photo sharing, data is fed into a system as a stream $(v_1, v_2, \dots.)$ and under a particular order. At any time point $n$, the user can request a summary of the $n$ elements $v_{1:n}\triangleq\{v_1, v_2, \dots, v_n\}$ he/she has seen so far. The greedy algorithm and its variants are not appropriate to the streaming setting both for memory and computational reasons, i.e., they require storing all $n$ elements in advance, and computing their marginal gains each step. In this paper, we study how to solve (\ref{equ:card}) with $V=v_{1:n}$ for any $n$ in the streaming setting in one pass using a memory of size only $b+k\ll n$, where $b$ is the number of buffered elements and $k$ is the number of elements in the solution set. 

\subsection{Related Work}
\label{sec:related-work}

Various strategies have been proposed in previous work to solve (\ref{equ:card}) in the streaming setting. A thresholding algorithm in~\cite{Badanidiyuru} adds element $v$ to a summary $S$ if its marginal gain $f(v|S)$ exceeds a threshold $\tau=\frac{f(S^*)/2-f(S)}{k-|S|}$, where  $S^*\in\argmax_{S\subseteq v_{1:n},\\|S|\leq k}f(S)$ and $f(S^*)$ is the global maxima. One function evaluation is required per step for computing $f(v|S)$. However, $f(S^*)$ in $\tau$ is not known in advance for a stream so the proposed sieve-streaming algorithm starts by running multiple instances of the thresholding algorithm with different estimates of $f(S^*)$, and dynamically removes the instances whose estimates of $f(S^*)$ lie outside the interval updated by the maximal singleton gain. At the end, the instance achieving the largest $f(S)$ is used for the  solution $S$. It has a guarantee of $f(S)\geq \left(\frac{1}{2}-\epsilon\right)f(S^*)$ with $O(k\log k/\epsilon)$ memory. A sliding window method based on thresholding \cite{ChenNZ16} has also been proposed that emphasizes recent data.

Swapping between new elements and the ones in $S$ is a natural yet more computationally expensive strategy~\cite{Buchbinder, Chekuri, gomes10budgeted}. The algorithm initializes $S$ with the first $k$ elements from the stream, and keeps replacing a new element $v$ and $u\in S$ once $f(S\cup v\backslash u)\geq (1+c/k)f(S)$ \cite{Buchbinder} or $f(v|S)\geq\alpha+(1+\beta)f(u|S_u)$ \cite{Chekuri}, where $c$, $\alpha$ and $\beta$ are nonnegative constants, and $S_u$ denotes the historical solution set $S$ right before adding $u$ to it. Both cases have guarantee $f(S)\geq \frac{1}{4}f(S^*)$ (when $c=1$ for the former) with memory size $k$. The latter requires less computation, i.e, one function evaluation per element, comparing to $k$ evaluations required by the former.

A mini-batch based strategy splits the whole stream evenly into $k$ segments, and sequentially adds to $S$ the element $v$ with the largest marginal gain $f(v|S)$ in each segment. It was introduced via the submodular secretary problem and its extensions \cite{Bateni}. This algorithm has an approximation bound of $f(S)\geq \frac{1-1/e}{e}f(S^*)$ in expectation with memory size $k$, if the data arrives in a uniformly at random order. This method requires only one function evaluation per element, but it needs to know the length of the stream $n$ in advance, impossible when the stream is unboundedly large and a summary can be requested at any time.

A hardness result is given in Theorem 1.6 of \cite{Buchbinder}: for solving (\ref{equ:card}) in the online setting, there is no deterministic algorithm $1/2+\epsilon$-competitive for any constant $\epsilon>0$. In Lemma 4.7 of \cite{Buchbinder} (Lemma 4.11 in its arXiv version), the approximation factor in the worst case cannot exceed $1/2+\epsilon$ unless $b\geq n-k$ and all the $n$ elements up to a summary request is stored in the memory. Note the online setting in \cite{Buchbinder}\footnote{The online setting in \cite{Buchbinder} is, and we quote: ``The elements of N arrive one by one in an online fashion. Upon arrival, the online algorithm must decide whether to accept each revealed element into its solution and this decision is irrevocable.''} is slightly different from our streaming setting in that it does not allow the buffering of unselected elements. However, it is trivial to generalize the $1/2$-hardness to algorithms with buffer size $b\leq n-k$. In particular, we consider the submodular function used in the proof of Lemma 4.7 in \cite{Buchbinder}, and use their notations for $u$ and $v$: the hardness stays $1/2$ unless the algorithm buffers at least one $u$, but since the algorithm cannot distinguish $u$ and $v$ until seeing the last element $w$, it needs to buffer at least $n-k+1$ elements to ensure that one $u$ is stored in the buffer. 

Different settings for streaming submodular maximization have also been studied recently. A robust streaming algorithm \cite{robustStream} has been studied for when the data provider has the right to delete at most $m$ elements due to privacy concerns. Given any single-pass streaming algorithm with an $\alpha$-approximation guarantee, it runs a cascading chain of $r$ instances of such an algorithm with non-overlapping solutions to ensure that only one solution is affected by a deletion. Its solution still satisfies a $\alpha$-approximation guarantee when $m$ deletions are allowed. Another popularly studied setting is submodular maximization with sliding windows \cite{slidingWindow}, which aims to maintain a solution that takes only the last $W$ items into account.

In the present paper, we mainly focus on the classical streaming
setting where deletion or sliding windows is not considered.  Our
method, however, can be applied as a streaming algorithm subroutine in
the deletion-robust setting of \cite{robustStream}.



\subsection{Our Approach}

In practice, the thresholding algorithm must try a large number of thresholds $\tau$
(associated with different estimates of $f(S^*)$) to obtain a sufficiently good solution, because the solution set is sensitive to tiny changes in threshold $\tau$. This results in a high memory load. Though swapping and mini-batch strategies ask for a smaller memory size $k$, the former requires $k$ function evaluations per step, while the latter needs to know $n$ in advance and requires uniformly at random ordered elements, which cannot be justified in a streaming setting. Although the worst-case approximation factors of the three algorithms are $1/2$, $1/4$ and $(1-1/e)/e$ respectively, they perform much poorer in practice than the offline greedy algorithm, which has the worst-case approximation factor $1-1/e$ but usually performs much better than $1-1/e$.


The main contributions of this paper is a novel streaming algorithm (that
we call ``stream clipper'') that can achieve similar empirical performance to the offline greedy algorithm, and we analyze when this is the case. It is given in Algorithm~\ref{alg:StreamBuffer} and illustrated in the left plot of Figure~\ref{fig:sc}. It uses two thresholds $\tau^-$ and $\tau^+\geq \tau^-$ to process each element $v$: it adds $v$ to the solution set $S$ if $f(v|S)\geq \tau^+$; rejects $v$ if $f(v|S)\leq\tau^-$; otherwise (i.e., $f(v|S)\in(\tau^-,\tau^+)$) places $v$ in a buffer $B$. The final solution is generated by a greedy algorithm starting from the obtained $S$ and adds more elements from $B$ to $S$ until $S$ reaches the budget size $k$. Since the elements with marginal gains slightly less than $\tau^+$ are saved in $B$ and given a second chance to be selected into $S$, the two-threshold scheme mitigates the instability of a single thresholding method without requiring the testing of a large number of different thresholds simultaneously.

According to the hardness analysis in \cite{Buchbinder}, the worst-case approximation factor of stream clipper cannot exceed $1/2$ for memory size $b<n-k$. However, we explicitly show that in some cases when thresholds $\tau^-$ and $\tau^+$ fulfill certain data dependent conditions, its approximation factor lies in $[1/2, 1-1/e]$. In addition, given $\tau^-$, $\tau^+$ and a data stream to process, we show simple conditions to justify when stream clipper can guarantee an approximation factor $1-\alpha$ for any $\alpha\in[0, 1/2]$. 

An advanced version of stream clipper is given in Algorithm~\ref{alg:StreamClipper} with illustration in the right plot in Figure~\ref{fig:sc}. It allows an element in buffer $B$ to replace some element in $S$, if such swapping improves the objective $f(S)$. This avoids extra computation spent on swapping for every new element $v\in V$. In addition, the advanced version adapts thresholds to remove elements from the buffer once its size exceeds a user defined limit $b$. This guarantees memory efficiency even for a poor initialization of the thresholds. In Section \ref{sec:exp}, experiments on news and video summarization show that stream clipper significantly outperforms other streaming algorithms consistently (Figure~\ref{fig:nchange_box}-\ref{fig:NYT_utility}, Figure~\ref{fig:videoF1}). In most experiments, it achieves $f(S)$ as large as the offline greedy algorithm, and produces a summary of similar quality, but costs much less memory and computation due to its streaming setting.

\section{Stream Clipper}

In the following, we first introduce a na\"{i}ve stream clipper and then later its advanced version with swapping, threshold adaptation, and buffer cleaning procedures. Detailed analysis of the approximation bound in different cases (rather then the worst case) for the na\"{i}ve version follows. We further show the analysis can be extended to the advanced version. In the following, we use the letters ``A'' for Algorithm and ``L'' for line. For example, A\ref{alg:StreamBuffer}.L2-5 refers to Lines 2-5 of Algorithm~\ref{alg:StreamBuffer}. 


\subsection{Na\"{i}ve Stream Clipper}

\begin{algorithm}
	\SetKwInOut{Input}{Input}\SetKwInOut{Output}{Output}\SetKwInOut{Initialize}{Initialize}
	\DontPrintSemicolon
	\Input{$(v_1,v_2,\dots,v_n)$, $k$, $\tau^-$, $\tau^+$}
	\Output{$S$}
	\Initialize{$S\leftarrow\emptyset$, $B\leftarrow\emptyset$}
	\BlankLine
	\For{$i\leftarrow 1$ \KwTo $n$}{
		\uIf{$f(v_i|S)\geq \tau^+$ and $|S|<k$}{
			$S\leftarrow S\cup v_i$\;
		}
		\uElseIf{$f(v_i|S)> \tau^-$}{
			$B\leftarrow B\cup v_i$
		}
		\Else{Reject $v_i$
		}
	}
	\While{$|S|<k$}{
		$v^*=\argmax_{v\in B} f(v|S)$\;
		$S\leftarrow S\cup v^*$, $B\leftarrow B\backslash v^*$\;
	}
	\caption{\texttt{na\"{i}ve\_stream\_clipper}}
	\label{alg:StreamBuffer}
\end{algorithm}

We first give a na\"{i}ve version of stream clipper in Algorithm~\ref{alg:StreamBuffer}. It selects element $v$ if $f(v|S)\geq \tau^+$ and $|S|<k$, and stores $f(v|S)$ in $s(v)$ (A\ref{alg:StreamBuffer}.L2-3), while rejects $v$ if $f(v|S)\leq \tau^-$ (A\ref{alg:StreamBuffer}.L7). It places $v$ whose marginal gain is between $\tau^-$ and $\tau^+$ (A\ref{alg:StreamBuffer}.L4) into the buffer $B$ (A\ref{alg:StreamBuffer}.L5). Once a summary is requested, a greedy algorithm (A\ref{alg:StreamBuffer}.L8-10) adds more elements from $B$ to $S$ until $|S|=k$. 

In the following, we use $S_i$ and $B_i$ to represent $S$ and $B$ at the end of the $i^{th}$ iteration of the for-loop in Algorithm~\ref{alg:StreamBuffer}. Note $S_n$ and $B_n$ are the solution $S$ and buffer $B$ after passing $n$ elements but before running greedy procedure in  A\ref{alg:StreamBuffer}.L8-10. We use $S_{sc}$ to represent the final solution of Algorithm~\ref{alg:StreamBuffer}, use $m$ for the size of $S_n$, and use $v_{s_i}$ to denote the $i^{th}$ selected element by A\ref{alg:StreamBuffer}.L3. In above algorithm, the thresholds $\tau^-$ and $\tau^+$ are fixed, so tuning them is important for getting a good solution. However, in the advanced version introduced below, they are updated adaptively with the incoming data stream, and thus more robust to the initialization values.

\subsection{Advanced Stream Clipper}

\begin{algorithm}[h]
	\SetKwInOut{Input}{Input}\SetKwInOut{Output}{Output}\SetKwInOut{Initialize}{Initialize}
	\DontPrintSemicolon
	\Input{$(v_1, v_2, \dots, v_n)$, $k$, $b$, $\hat f(S^*)$}
	\Output{$S$}
	\Initialize{$S\leftarrow\emptyset$, $B\leftarrow\emptyset$, $\Delta \tau=\frac{\hat f(S^*)}{20k}$, $\tau^-=\frac{\hat f(S^*)}{2k}-\Delta \tau$, $\tau^+=\frac{\hat f(S^*)}{2k}+\Delta \tau$}
	\BlankLine
	\For{$i\leftarrow 1$ \KwTo $n$}{
		\uIf{$f(v_i|S)\geq \tau^+$ and $|S|<k$}{
			$S\leftarrow S\cup v_i$
		}
		\uElseIf{$f(v_i|S)> \tau^-$}{
			$u\in\argmax_{w\in S} f(S\backslash w\cup v_i)$\;
			$\rho=[f[(S\backslash u)\cup v_i]-f(S)]/|S|$\;
			\eIf{$\rho>0$}{
				$\tau^-\leftarrow\tau^-+f(u|S\backslash u\cup v_i)$\;
				$\tau^+\leftarrow\tau^++\rho$\;
				$S\leftarrow S\backslash u\cup v_i$
			}{
				$B\leftarrow B\cup v_i$
			}			
		}
		\Else{Reject $v_i$
		}
		\While{$|B|=b$}{
			$\tau^-\leftarrow \min\{\tau^-+\Delta \tau, \tau^+\}$\;
			$B\leftarrow B\backslash\{v\in B:f(v|S)\leq \tau^-\}$
		}
	}
	\While{$|S|<k$}{
		$v^*\in\argmax_{v\in B} f(v|S)$\;
		$S\leftarrow S\cup v^*$, $B\leftarrow B\backslash v^*$\;
	}
	\caption{\texttt{stream\_clipper}}
	\label{alg:StreamClipper}
\end{algorithm}

In practice, we develop two additional strategies to (1) achieve further improvement by occasional  swapping between buffered element in $B$ and element in solution $S$, and (2) keep the buffer size $|B|\leq b$ by removing unimportant elements from $B$. The advanced version of stream clipper after applying these two strategies is given in Algorithm~\ref{alg:StreamClipper}, where A\ref{alg:StreamClipper}.L5-10 denotes the first strategy, and A\ref{alg:StreamClipper}.L15-17 denotes the second strategy. Algorithm~\ref{alg:StreamClipper} is the same as Algorithm~\ref{alg:StreamBuffer} if we ignore these steps.


The swapping procedure in A\ref{alg:StreamClipper}.L5-10 is applied only to the new element $v_i$ whose marginal gain is between $\tau^-$ and $\tau^+$. A\ref{alg:StreamClipper}.L5 computes the objective $f[(S\backslash w)\cup v_i]$ for all the possible swappings between $v_i$ and element $w\in S$, and finds $u\in S$ achieving the maximal objective $f[(S\backslash u)\cup v_i]$. A\ref{alg:StreamClipper}.L6 computes $\rho$, the average of the swapping gain on the objective over all $|S|$ elements in $S$. If $\rho>0$, which means swapping brings positive improvements to the objective, the swapping is committed as in A\ref{alg:StreamClipper}.L10. Comparing to previous swapping methods \cite{Buchbinder} that computes $f[(S\backslash w)\cup v_i]$ for all new element $v_i$, stream clipper only computes A\ref{alg:StreamClipper}.L5 for $v_i$ such that $f(v_i|S)\in(\tau^-,\tau^+)$. This improves the efficiency since computing $A\ref{alg:StreamClipper}.L5$ requires $|S|$ function evaluations.

When the buffer size reaches the user defined limit $b$, stream clipper increases $\tau^-$ by step size $\Delta \tau$ as shown in A\ref{alg:StreamClipper}.L16. Since the lower threshold $\tau^-$ increases, elements in buffer $B$ whose marginal gain $f(v|S)\leq\tau^-$ can be removed from $B$ (A\ref{alg:StreamClipper}.L17). We repeat this buffer cleaning procedure until $|B|<b$. Note the maximal value of $\tau^-$ after it increases is $\tau^+$, because $|B|=0$ if $\tau^-=\tau^+$. 


In Algorithm~\ref{alg:StreamClipper}, parameter $\hat f(S^*)$ is an estimate to $f(S^*)$. In practice, it can be initialized as $f(v_1)$ and increased to $f(S)$ according to solution set $S$ achieved in later steps. We initialize the ``step size'' $\Delta\tau$ as $\hat f(S^*)/20k$ since it works well empirically. The two thresholds are initialized as shown in Algorithm~\ref{alg:StreamClipper}. Note we can start with a sufficiently small $\tau^-$ is to guarantee $|B_n|\geq k-|S_n|$ and $\tau^+\geq \tau^-$, and adaptively increase it later as in A \ref{alg:StreamClipper}.L16.

\subsection{Approximation Bound}\label{sec:bound}

We study the approximation bound of Algorithm~\ref{alg:StreamBuffer} in different cases rather than the worst case. Firstly, we assume $\tau^-$ is properly selected so $|B_n|\geq k-|S_n|$. This guarantees $k-|S_n|$ elements are selected into $S$ by the greedy algorithm in A\ref{alg:StreamBuffer}.L8-10 and thus there are $k$ elements in the final output $S_{sc}$. A trivial choice of $\tau^-$ is $0$.
\begin{lemma}\label{lemma:lt0}
	If $\tau^-=0$ and $\tau^+\geq \tau^-$, then $|B_n|\geq k-|S_n|$ before A\ref{alg:StreamBuffer}.L8.
\end{lemma}
When $\tau^-=0$, all the elements whose marginal gain is less than $\tau^+$ will be stored in the buffer, and may lead to a large $B_n$. Note the advanced version Algorithm~\ref{alg:StreamClipper} can start from $\tau^-=0$, and adaptively increase it and clean the buffer when $|B|$ exceeds the limit $b$. By following similar proof technique in \cite{greedyapprox},
we have the theorem below. Please refer to \cite{Supp} for its proof.
\begin{theorem}\label{the:k1Snbound}
	If submodular function $f(\cdot)$ is monotone non-decreasing and normalized, let $k_n\triangleq|(S^*\backslash S_n)\backslash B_n|$, the following result holds for the final output $S$ of Algorithm~\ref{alg:StreamBuffer}.
	\begin{equation}\label{equ:k1Snbound}
	f(S_{sc})\geq \left(1-e^{-\frac{k-|S_n|}{k-k_n}}\right)\left(f(S^*)-k_n\tau^-\right)+e^{-\frac{k-|S_n|}{k-k_n}}|S_n|\tau^+,
	\end{equation}
\end{theorem}

The bound in (\ref{equ:k1Snbound}) is a convex combination of $f(S^*)-k_n\tau^-$ and $|S_n|\tau^+$. It depends on $k_n$, $S_n$, $\tau^+$, $\tau^-$ and $f(S^*)$: $|S_n|$ is known once a summary is requested; thresholds $\tau^+$ and $\tau^-$ are pre-defined parameters; $f(S^*)$ is the optimum we need to compare to. However, $k_n$ is the number of elements from optimal set $S^*$ that have been rejected by A\ref{alg:StreamBuffer}.L7. It depends on $S^*$ that may not be known. In order to remove the dependency on $k_n$, we take the minimum of the right hand side of (\ref{equ:k1Snbound}) over all possible values of $k_n\in[0, k]$. We use $g(k_n)$ to denote the right hand side of (\ref{equ:k1Snbound}),
\begin{equation}
g(k_n) = \left(1-e^{-\frac{k-|S_n|}{k-k_n}}\right)\left(f(S^*)-k_n\tau^-\right)+e^{-\frac{k-|S_n|}{k-k_n}}|S_n|\tau^+.
\end{equation}
Since $g(k_n)$ has a complex shape, we firstly study its first and second order derivatives.

\begin{lemma}\label{lemma:derivative}
	The derivative and second order derivative of $g(k_n)$ are
	\begin{equation}\label{equ:derivative}
	\begin{array}{ll}
	\frac{\partial g}{\partial k_n}=&\left[e^{-\frac{k-|S_n|}{k-k_n}}\left(1-\frac{k_n(k-|S_n|)}{(k-k_n)^2}\right)-1\right]\tau^-+\\&e^{-\frac{k-|S_n|}{k-k_n}}\frac{k-|S_n|}{(k-k_n)^2}\left[f(S^*)-|S_n|\tau^+\right],
	\end{array}
	\end{equation}
	and
	\begin{equation}\label{equ:2derivative}
	\begin{array}{ll}
	\frac{\partial^2 g}{\partial k_n^2}=&e^{-\frac{k-|S_n|}{k-k_n}}\frac{k-|S_n|}{(k-k_n)^2}\left[(k-2k_n+|S_n|)\right.\\
	&\left.(f(S^*)-|S_n|\tau^+)-(2k^2-3kk_n+|S_n|k_n)\tau^-\right].
	\end{array}
	\end{equation}
\end{lemma}

\begin{proposition}\label{prop:boundmin}
	When $f(S^*)\geq k_n\tau^-+|S_n|\tau^+$, the minimum value $\min_{k_n\in[0, k]}g(k_n)$ of the bound $g(k_n)$ given in (\ref{equ:k1Snbound}) w.r.t. $k_n$ is either $g(k) = f(S^*)-k\tau^-$, or
	$g(0) = \left(1-e^{-1+\frac{|S_n|}{k}}\right)f(S^*)+e^{-1+\frac{|S_n|}{k}}|S_n|\tau^+$.
\end{proposition}

By using Proposition \ref{prop:boundmin}, we can derive the minimum value of $g(k_n)$ in three different cases, which corresponds to three ranges of $f(S^*)$ determined by $\tau^+$, $\tau^-$ and $|S_n|$. This leads to the following theorem.

\begin{theorem}\label{the:Snbound}
	Under the assumptions of Theorem \ref{the:k1Snbound}, we have
	
	Case 1: when $f(S^*)\leq k_n\tau^-+|S_n|\tau^+$,
	\begin{equation}\label{equ:bound1}
	f(S_{sc})\geq |S_n|\tau^+\geq \frac{|S_n|}{k}\times \frac{f(S^*)}{2}.
	\end{equation}
	Case 2: when $k_n\tau^-+|S_n|\tau^+< f(S^*)< e^{1-\frac{|S_n|}{k}}k\tau^-+|S_n|\tau^+$,
	\begin{equation}\label{equ:bound2}
	f(S_{sc})\geq f(S^*)-k\tau^-.
	\end{equation}
	Case 3: when $f(S^*)\geq e^{1-\frac{|S_n|}{k}}k\tau^-+|S_n|\tau^+$,
	\begin{equation}\label{equ:bound3}
	f(S_{sc})\geq \left(1-e^{-1+\frac{|S_n|}{k}}\right)f(S^*)+e^{-1+\frac{|S_n|}{k}}|S_n|\tau^+.
	\end{equation}
\end{theorem}
\textbf{Remarks:} In case 1, when $\tau^-=\tau^+=f(S^*)/(2k)$, buffer $B=\emptyset$ and Algorithm~\ref{alg:StreamBuffer} reduces to sieve-streaming \cite{Badanidiyuru}, so the bound is $(1/2)f(S^*)$. In the following corollary, we further show in cases 2 \& 3, better (i.e., $\geq 1/2$) bounds can be achieved when $|S_n|<k$, since the greedy algorithm in the end of Algorithm~\ref{alg:StreamBuffer} further takes advantage of elements from buffer $B$.

\begin{corollary}\label{cor:bestbound}
	Under the assumptions of Theorem \ref{the:k1Snbound}, when $f(S^*)\leq k_n\tau^-+|S_n|\tau^+$ (case-1), if $\tau^-\leq \frac{f(S^*)}{2k_n}$ and $\tau^+\geq \tau^-$, $f(S_{sc})\geq f(S^*) /2$. When $f(S^*)> k_n\tau^-+|S_n|\tau^+$ (case-2\&3), if $|S_n|=0$, $f(S_{sc})\geq \left(1-e^{-1}\right)f(S^*)$.
\end{corollary} 
According to Corollary \ref{cor:bestbound}, although the approximation factor is possible to be  $1-e^{-1}$ for cases 2 \& 3, the worst case bound is still $f(S^*)/2$. This obeys the $1/2$ hardness given in \cite{Buchbinder}, i.e., it is impossible to improve the worst-case bound over $f(S^*)/2$. However, the bound can be strictly better than $f(S^*)/2$ on specific orders of the same set of elements $v_{1:n}$. Given thresholds $\tau^-$ and $\tau^+$, for a data stream with a specific order and an $\alpha\in[0, 1/2]$, we give the conditions to justify whether stream clipper can achieve an approximation factor $1-\alpha$.



In the following analysis, we use $\sigma=(\sigma_1, \sigma_2, \cdots, \sigma_n)$, a sequence of $n$ distinct integers from $1$ to $n$, to denote the order of elements $v_{1:n}$ in the stream, i.e., $(v_{\sigma_1}, v_{\sigma_2}, \cdots, v_{\sigma_n})$. We use $\Sigma$ to represent the set of all $n!$ orders. By analyzing the three cases in Theorem \ref{the:Snbound}, we can locate $\tau^-$ and $\tau^+$ in specific ranges. In each range, we characterize the orders on which $f(S_{sc})\geq(1-\alpha)f(S^*)$ and buffer size is bounded by $b$, i.e., $|B_i|\leq b~ \forall i\in[n]$.

\begin{proposition}\label{the:combine23}
	1) For any $\alpha\in[0,1/e]$, given $\tau^-$ and $\tau^+\geq\tau^-$ to use in stream clipper (Algorithm~\ref{alg:StreamBuffer}), define $\notag m^{*}\triangleq\min M^{'}$, where
	\begin{align}\label{equ:caset1}
	M^{'}\triangleq\left\{m\in[n]:\tau^{+}\geq\frac{(1-e^{1-m/k}\alpha)f(S^{*})}{m},\tau^{-}\leq \frac{f(S^{*})-m\tau^+}{e^{1-m/k}k}\right\},
	\end{align}
	if $M^{'}\neq\emptyset$, for any order $\sigma\in\left\{\sigma\in\Sigma: |S_n|\geq m^*, |B_n|\leq b\right\}$, we have $f(S_{sc})\geq(1-\alpha)f(S^*)$.
	
	2) For any $\alpha\in(1/e,1/2]$, given $\tau^-$ and $\tau^+\geq\tau^-$ to use in stream clipper (Algorithm~\ref{alg:StreamBuffer}), define
	\begin{equation}\label{equ:caset3}
	\begin{array}{ll}
	&M_1\triangleq\left\{m\in[n]:\frac{(1-\alpha)f(S^*)}{k}\leq \tau^+\leq\frac{f(S^*)}{m+k}\right\},\\ &M_2\triangleq\left\{m\in[n]:\tau^-\leq \frac{f(S^*)-m\tau^+}{e^{1-m/k}k}\right\},
	\end{array}
	\end{equation}
	and
	\begin{align}\label{equ:case2t2} 
	\notag M\triangleq&\left\{m\in[n]:\frac{f(S^*)-e^{1-\frac{m}{k}}k\tau^-}{m}<\tau^+<\right.\\
	&~~~~~~~~~~~~~~~~~~~\left.\frac{f(S^*)-k\tau^-}{m}, \tau^-\leq \frac{\alpha f(S^*)}{k}\right\}
	\end{align} 
	for any order
	$\sigma\in\bigcup\limits_{m\in (M_1\cap M_2)\cup M}\left\{\sigma\in\Sigma: |S_n|=m, |B_n|\leq b\right\}$,
	we have $f(S_{sc})\geq(1-\alpha)f(S^*)$.
\end{proposition}

The detailed proof is given in \cite{Supp}. In the advanced version of stream clipper, we can adjust $\tau^-$ and $\tau^+$ to guarantee an nonempty $M^{'}$. The conditions in (\ref{equ:caset1}) can provide some clues of how to adjust them based on the updated estimate of $f(S^*)$. According to Proposition \ref{the:combine23}, given $\tau^-$, $\tau^+\geq\tau^-$, and any $\alpha\in[0,1/2]$, for the orders on which stream clipper achieves 1) $|S_n|\geq m^*$ and $|B_n|\leq b$ when $\alpha\in[0,1/e]$, or 2) $|S_n|=m$ and $|B_n|\leq b$ for every $m\in (M_1\cap M_2)\cup M$ when $\alpha\in(1/e,1/2]$, we have $f(S_{sc})\geq(1-\alpha)f(S^*)$ with $(1-\alpha)\in[1-1/e, 1]$.

\textbf{Remarks:} We can easily extend the above analysis of Algorithm~\ref{alg:StreamBuffer} to Algorithm~\ref{alg:StreamClipper} by replacing $\tau^-$ and $\tau^+$ in them with $\tau_n^-$ and $\tau_n^+$ (the thresholds after step $n$) respectively. Details are given in \cite{Supp}.

\begin{figure}
	\begin{center}
		\includegraphics[width=1\linewidth]{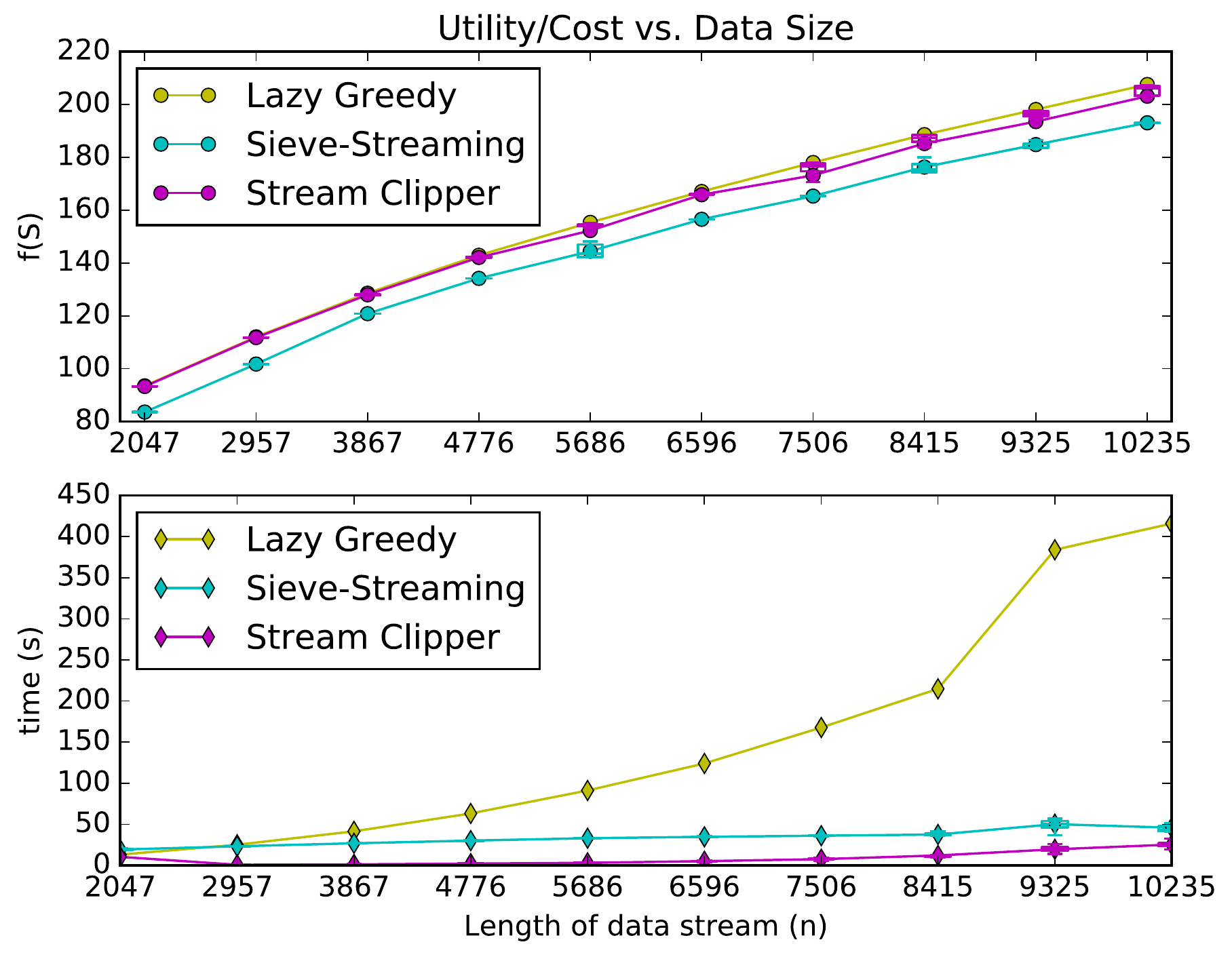}
	\end{center}
	\caption{Utility $f(S)$ and time cost vs. length of data stream $n$ on the same data of $20$ different random orders. Stream clipper achieves similar utility as offline greedy, but has computational costs similar to other streaming algorithms (i.e., much less than the offline greedy).}
	\label{fig:nchange_box}
\end{figure}

\begin{figure}
	\begin{center}
		\includegraphics[width=1\linewidth]{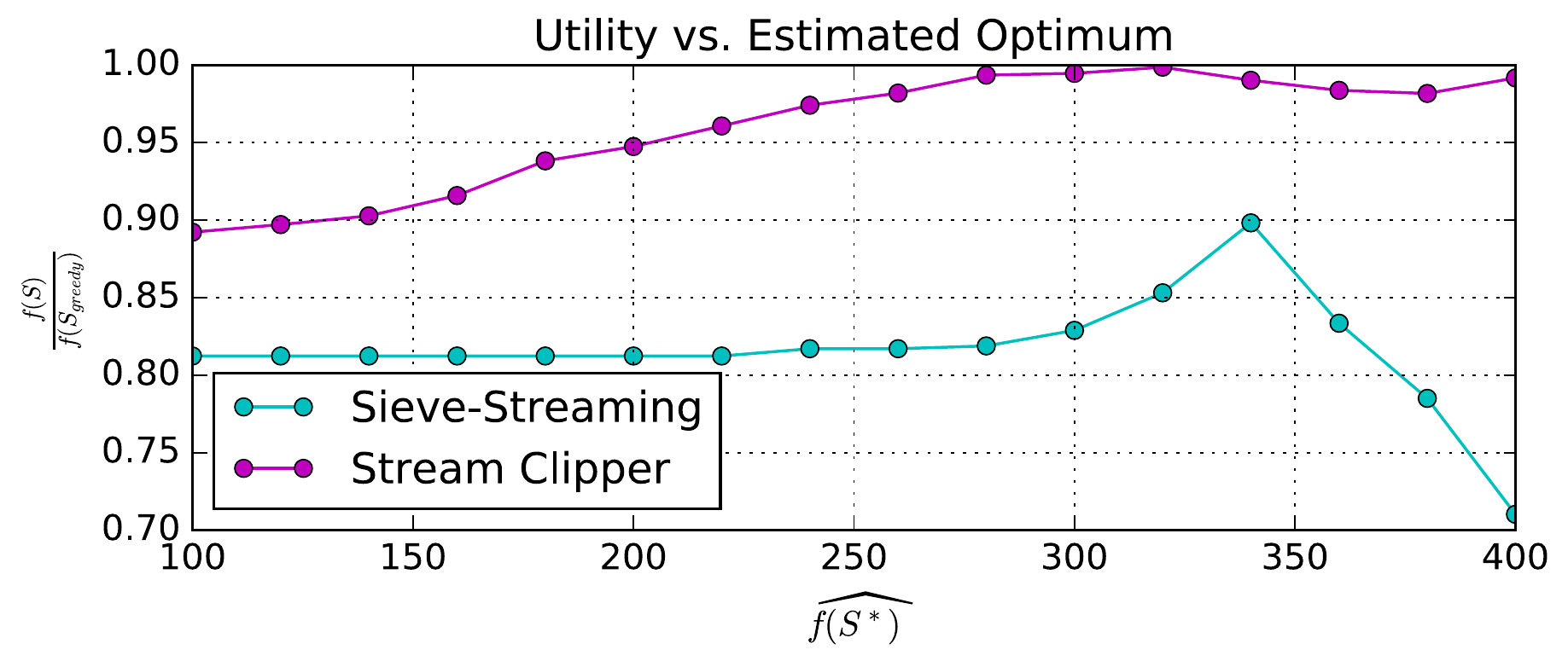}
	\end{center}
	\caption{Relative utility $f(S)/f(S_{greedy})$ vs. $\hat f(S^*)$. Different from sieve-streaming, stream clipper does not heavily rely on an accurate estimate $\hat f(S^*)$ to guarantee a large utility, because it can adaptively tune the thresholds on the fly.}
	\label{fig:OPTchange}
\end{figure}

\begin{figure}[tp]
	\begin{center}
		\includegraphics[width=1\linewidth]{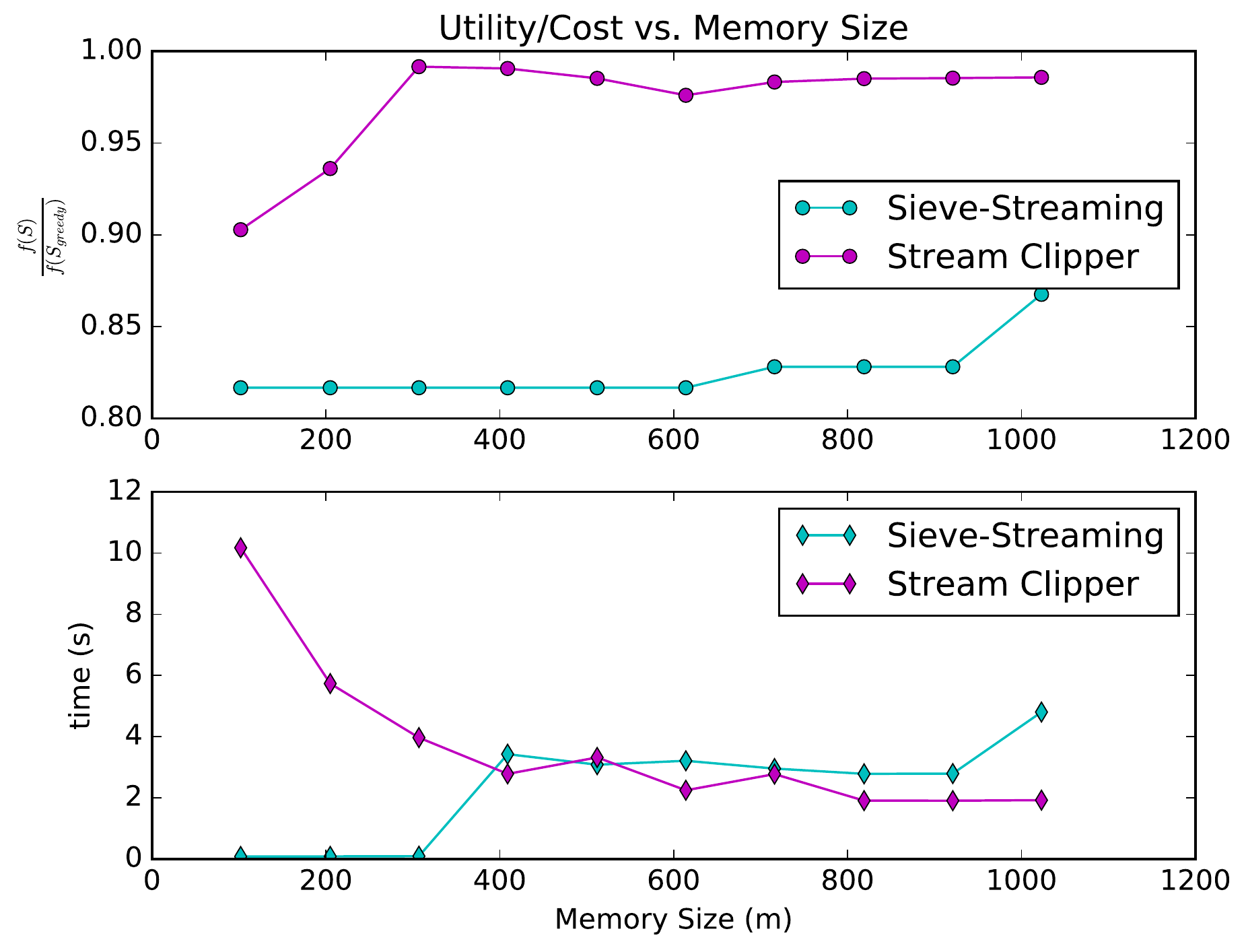}
	\end{center}
	\caption{Relative utility $f(S)/f(S_{greedy})$ and time cost vs.\ memory size $m$. It shows the advantage of stream clipper via the trade-off between memory usage and utility value. Stream clipper needs only to buffer $300$ sentences out of the $10235$ in the whole stream to obtain almost the same utility as the offline greedy procedure, while its time cost is similar to the other streaming algorithm. }
	\label{fig:mchange}
\end{figure}

\section{Experiments}\label{sec:exp}

In this section, on several news and video datasets, we compare summaries generated by stream clipper and other algorithms. We use the feature based submodular function \cite{featurebased} $f(S)=\sum_{u\in \mathcal U}\sqrt{c_u(S)}$ as our objective, where $\mathcal U$ is a set of features, and $c_u(S)=\sum_{v\in S} \omega_{v, u}$ is a modular score ($\omega_{v, u}$ is the affinity of element $v$ to feature $u$). This function typically achieves good performance on summarization tasks. Our baseline algorithms are the lazy greedy approach~\cite{lazygreedy} (which has identical output as greedy but is faster) and the ``sieve-streaming''~\cite{Badanidiyuru} approach for streaming submodular maximization, which has low memory requirements as it takes one pass over the data. Note in summarization experiments, a difference of $10\%$ on utility usually leads to large gap on rouge-2 and F1-score.


\subsection{Empirical Study on News}

An empirical study is conducted on a ground set containing sentences from all NYT articles on a randomly selected date between 1996 and 2007, which are from the
NYTs annotated corpus 1996-2007 (\url{https://catalog.ldc.upenn.edu/LDC2008T19}). Figure~\ref{fig:nchange_box} shows how $f(S)$ and time cost varies when we change $n$. We set the budget size $k$ of the summary to be the number of sentences in a human generated summary. The buffer size $b$ of stream clipper is fixed to $200$, while the number of trials in sieve-streaming is $50$, leading to memory requirement of $50k$, which is much larger than $200+k$ of stream clipper. In order to test how performance varies with the order of stream, for each $n$, we run same experiment on $20$ different random orders of the same data.

\begin{figure}[h!]
	\vspace{-0.8em}
	\begin{center}
		\includegraphics[width=1\linewidth]{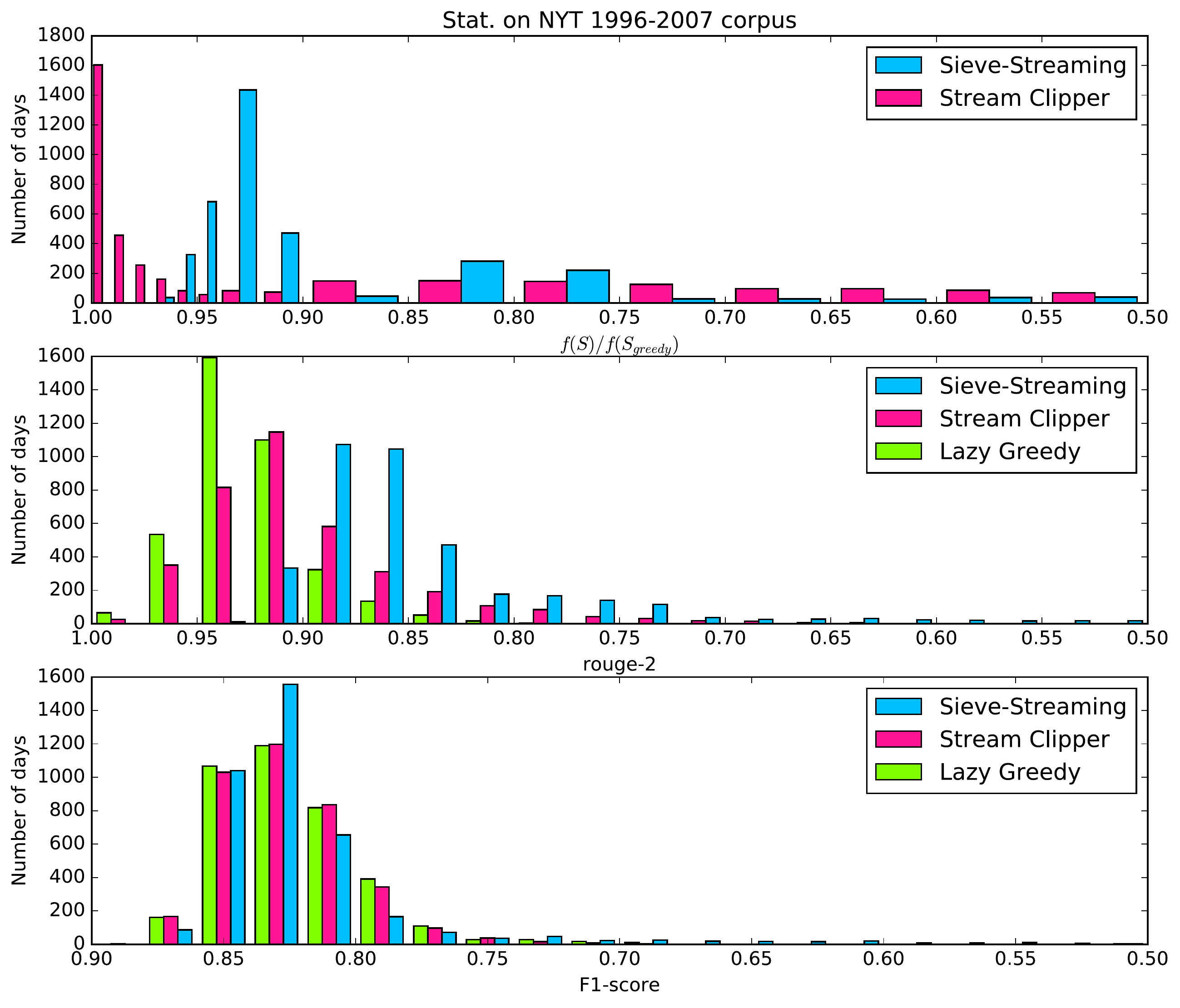}
	\end{center}
	\vspace{-0.8em}
	\caption{Statistics of relative utility $f(S)/f(S_{greedy})$, rouge-2 score and F1-score on daily news summarization results of $3823$ days' news from New York Times corpus between 1996-2007. Stream clipper achieves relative utility close to $1$ for most days. It has similar or more number of days than lazy greedy in the bins of high ($\geq 0.9$) rouge-2 and F1-score.}
	\label{fig:NYT_utility}
	\vspace{-0.8em}
\end{figure}
\begin{figure}[h!]
	\vspace{-0.8em}
	\begin{center}
		\includegraphics[width=1\linewidth]{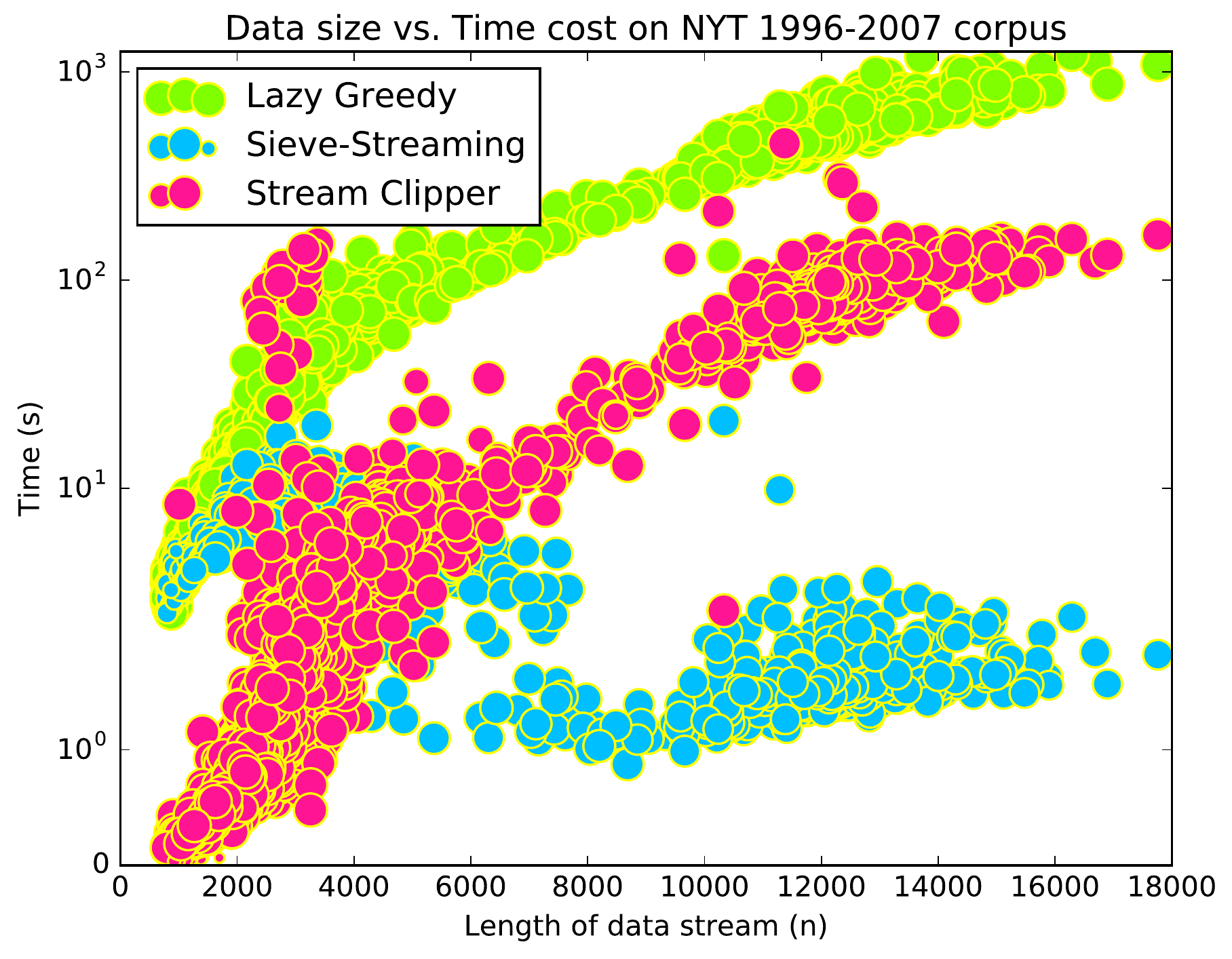}
	\end{center}
	\vspace{-0.8em}
	\caption{Length of data stream $n$ vs. time cost (exponential scale) on daily news summarization of $3823$ days' news from New York Times corpus between 1996-2007. The area of each circle is proportional to relative utility $f(S)/f(S_{greedy})$. The time cost of stream clipper grows slower than lazy greedy and saturates when $n\geq \sim14000$. The time cost of sieve-streaming increases at first, but becomes small and does not change after $n\geq\sim6000$. This is because the algorithm quickly fills $S$ with $k$ elements when $n<6000$ and does not change $S$ anymore. However, this avoids to enroll new elements and leads to worse relative utility reflected by the smaller blue circles.}
	\label{fig:NYT_nvstime}
	\vspace{-0.8em}
\end{figure}

The utility and time cost of both streaming algorithms do not change too much when the order changes. The utility curve of stream clipper overlaps that of lazy greedy, while its time cost is much less and increases more slowly than that of lazy greedy. Sieve-streaming performs much worse than SS in terms of utility, and its time cost is only slightly less  and even slightly decreases when increasing $n$ (this is because it quickly fills $S$ with $k$ elements and stops much earlier before seeing all $n$ elements). 



Figure~\ref{fig:mchange} shows how relative utility $f(S)/f(S_{greedy})$ ($S_{greedy}$ denotes the solution of the offline greedy algorithm) and time cost of the two streaming algorithms vary with memory size. Stream clipper quickly reaches a $f(S)\approx0.97f(S_{greedy})$ close to $f(S)$ of greedy algorithm once $m$ exceeds $200$, while sieve-streaming achieves much smaller $f(S)$ which does not increase until $m\geq 800$. Note the time cost of stream clipper is larger than that of sieve-streaming when $m\leq 400$ but dramatically decreases below it quickly. This is because the buffer cleaning procedure in A\ref{alg:StreamClipper}.L15-17 needs to be frequently executed if $m$ is small (and $b$ is small). However, a slight increase in memory size can effectively reduce the time cost.

Figure~\ref{fig:OPTchange} shows the robustness of the two streaming
algorithms to parameter $\hat f(S^*)$. In the wide range of
$[100, 400]$, stream clipper keeps a $\geq 0.9$ relative utility,
while sieve-streaming decreases dramatically around its peak value
$0.9$. Hence, sieve-streaming is more sensitive to $\hat f(S^*)$ and
thus a delicate search of $\hat f(S^*)$ is necessary. This results in
a high memory burden. By contrast, our approach adaptively adjusts
two thresholds via swapping and buffer cleaning even when the
estimate $\hat f(S^*)$ used to initialize them is inaccurate.

\subsection{NYT News Summarization}

In this section, we conduct summarization experiments on two news corpora, The New York Times annotated corpus 1996-2007 and the DUC 2001 corpus (\url{http://www-nlpir.nist.gov/projects/duc}). 

The first dataset includes all the articles published on The New York Times in $3823$ days from 1996-2007. For each day, we collect the sentences in articles associated with human generated summaries as the ground set $V$ (with sizes varying from $2000$ to $20000$), and extract their TFIDF features to build $f(S)$. We concatenate the sentences from all human generated summaries in the same date as reference summary. We compare the machine generated summaries produced by different methods with the reference summary by ROUGE-2 \cite{rouge} (recall on 2-grams) and ROUGE-2 F1-score (F1-measure based on recall and precision on 2-grams). We also compare their relative utility. As before, sieve-streaming holds a memory size of $50k$. Figure~\ref{fig:NYT_utility} shows the statistics over $3823$ days.

Stream clipper keeps a relative utility $\geq 0.95$ for most days, while sieve-streaming dominates the $\leq 0.95$ region. The ROUGE-2 score of stream clipper is usually better than sieve-streaming, but slightly worse than lazy greedy. However, its F1-score is very close to that of lazy greedy, while sieve-streaming's is much worse.


Figure~\ref{fig:NYT_nvstime} shows the number $n$ of collected sentences in each day and the corresponding time cost of each algorithm. The area of each circle is proportional to the relative utility. We use a log scale time axis for better visualization. Stream clipper is $10\sim 100$ times faster than lazy greedy. Their time cost have similar increasing speed, because as the summary size increases, the greedy stage in stream clipper tends to dominate the computation. The time cost of sieve-streaming decreases when $n\geq 6000$, but its relative utility also reduces fast. This is caused by the aforementioned early stopping.

\begin{figure}[htp]
	\begin{center}
		\includegraphics[width=1.0\linewidth]{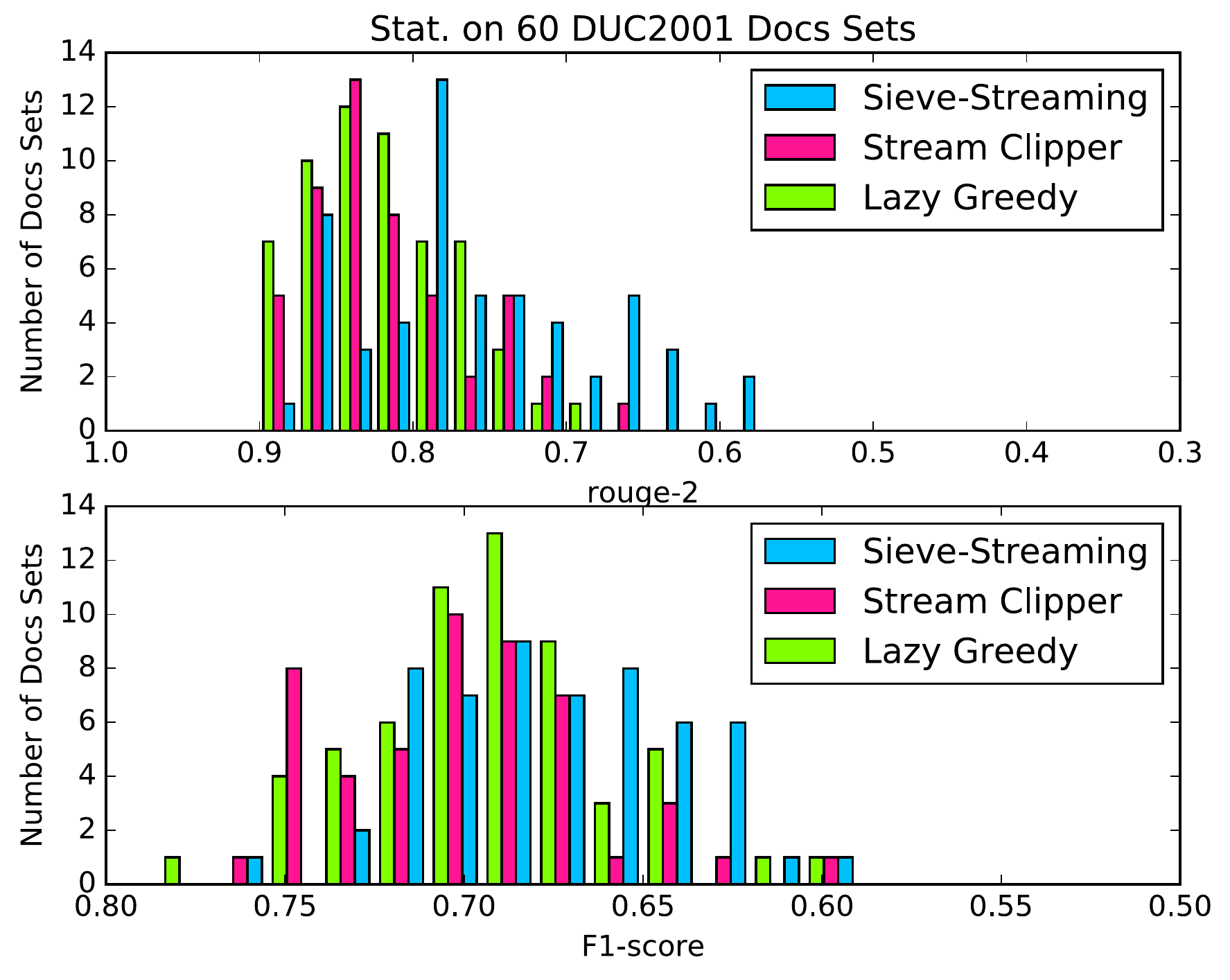}
	\end{center}
	\caption{Statistics of relative utility $f(S)/f(S_{greedy})$, rouge-2 score and F1-score on topic based news summarization results of $60$ document sets from DUC2001 training and test set, comparing to $400$-word human generated summary.}
	\label{fig:DUC2001_400_utility}
\end{figure}
\begin{figure}[htp]
	\begin{center}
		\includegraphics[width=1\linewidth]{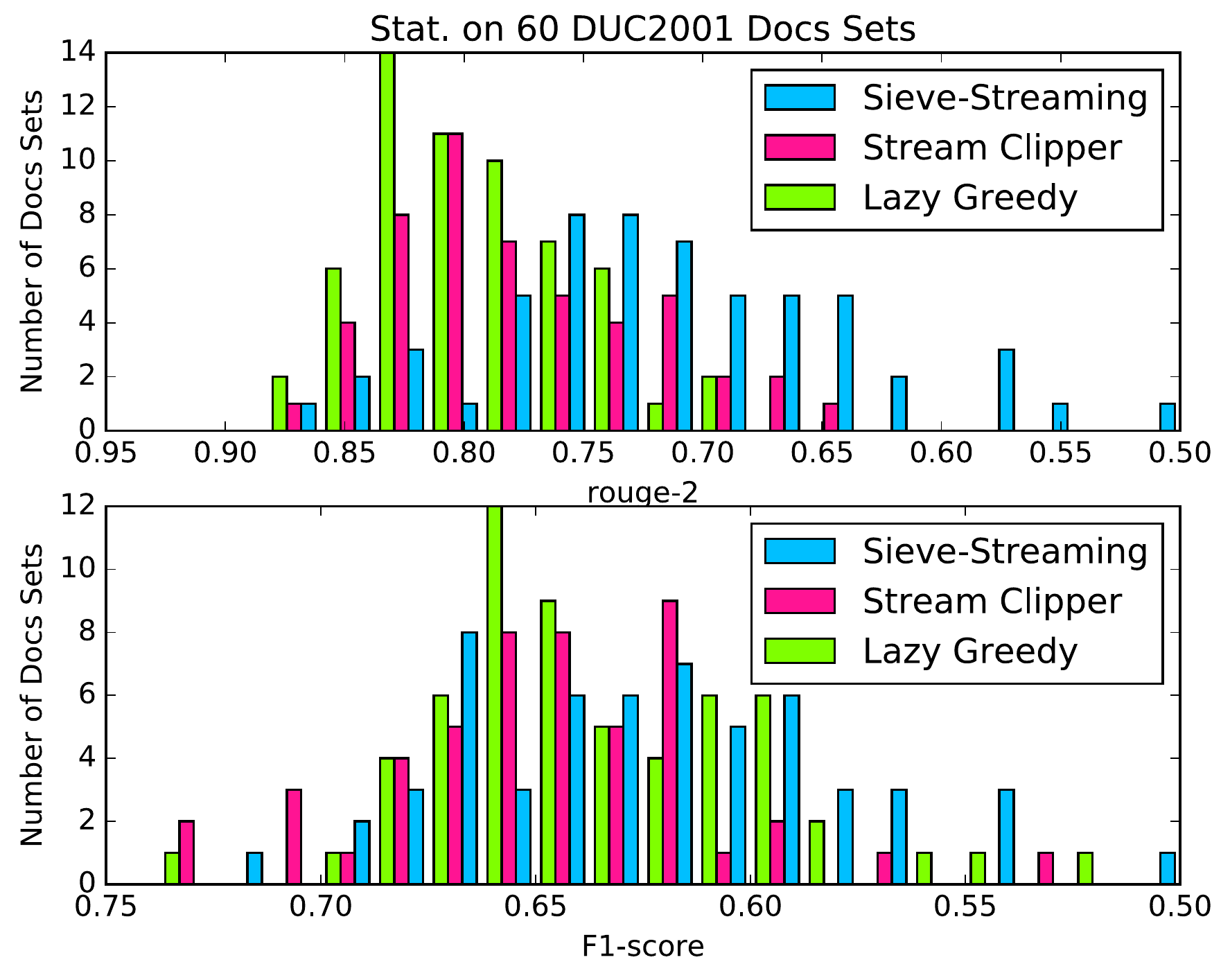}
	\end{center}
	\caption{Statistics of relative utility $f(S)/f(S_{greedy})$, rouge-2 score and F1-score on topic based news summarization results of $60$ document sets from DUC2001 training and test set, comparing to $200$-word human generated summary.}
	\label{fig:DUC2001_200_utility}
\end{figure}

\subsection{DUC2001 News Summarization}

\begin{table*}[htp]
	\caption{Performance of lazy greedy, sieve-streaming, and stream clipper on four topic summarization datasets from DUC 2001. For each topic, the machine generated summary is compared to four human generated ones having word count from 50 to 400.}
	\begin{center}
		\begin{tabular}{l|l|cc|cc|cc|cc}
			\hline
			\multirow{2}{*}{Algorithm} &\multirow{2}{*}{words} &\multicolumn{2}{c|}{Daycare} &\multicolumn{2}{c|}{Healthcare} &\multicolumn{2}{c|}{Pres92} & \multicolumn{2}{c}{Robert Gates}\\
			\cline{3-10}
			& &rouge2 &F1 &rouge2 &F1 &rouge2 &F1 &rouge2 &F1\\
			\hline
			\multirow{4}{*}{Lazy Greedy} &400 &$0.836$ &$0.674$ &$0.845$ &$0.686$ &$0.885$ &$0.686$ &$0.849$ &$0.734$\\
			&200 &$0.813$ &$0.615$ &$0.811$ &$0.632$ &$0.842$ &$0.623$ &$0.788$ &$0.682$\\
			&100 &$0.766$ &$0.542$ &$0.753$ &$0.605$ &$0.618$ &$0.420$ &$0.715$ &$0.621$\\
			&50 &$0.674$ &$0.484$ &$0.765$ &$0.539$ &$0.602$ &$0.341$ &$0.631$ &$0.514$\\
			\hline
			\multirow{4}{*}{Sieve-Streaming} &400 &$0.825$ &$0.687$ &$0.814$ &$0.711$ &$0.827$ &$0.710$ &$0.798$ &$0.745$\\
			&200 &$0.789$ &$0.627$ &$0.782$ &$0.675$ &$0.670$ &$0.659$ &$0.691$ &$0.688$\\
			&100 &$0.747$ &$0.542$ &$0.658$ &$0.597$ &$0.414$ &$0.443$ &$0.632$ &$0.620$\\
			&50 &$0.607$ &$0.475$ &$0.681$ &$0.551$ &$0.413$ &$0.345$ &$0.553$ &$0.477$\\
			\hline
			\multirow{4}{*}{Stream Clipper}  &400 &$0.841$ &$0.724$ &$0.838$ &$0.763$ &$0.859$ &$0.746$ &$0.834$ &$0.754$\\
			&200 &$0.803$ &$0.693$ &$0.807$ &$0.706$ &$0.810$ &$0.654$ &$0.764$ &$0.690$\\
			&100 &$0.763$ &$0.613$ &$0.778$ &$0.670$ &$0.581$ &$0.445$ &$0.732$ &$0.621$\\
			&50 &$0.689$ &$0.489$ &$0.794$ &$0.592$ &$0.463$ &$0.378$ &$0.628$ &$0.569$\\
			\hline
		\end{tabular}\label{table:DUC01}
	\end{center}
\end{table*}

\begin{figure*}[h!]
	\begin{center}
		\includegraphics[width=1\linewidth]{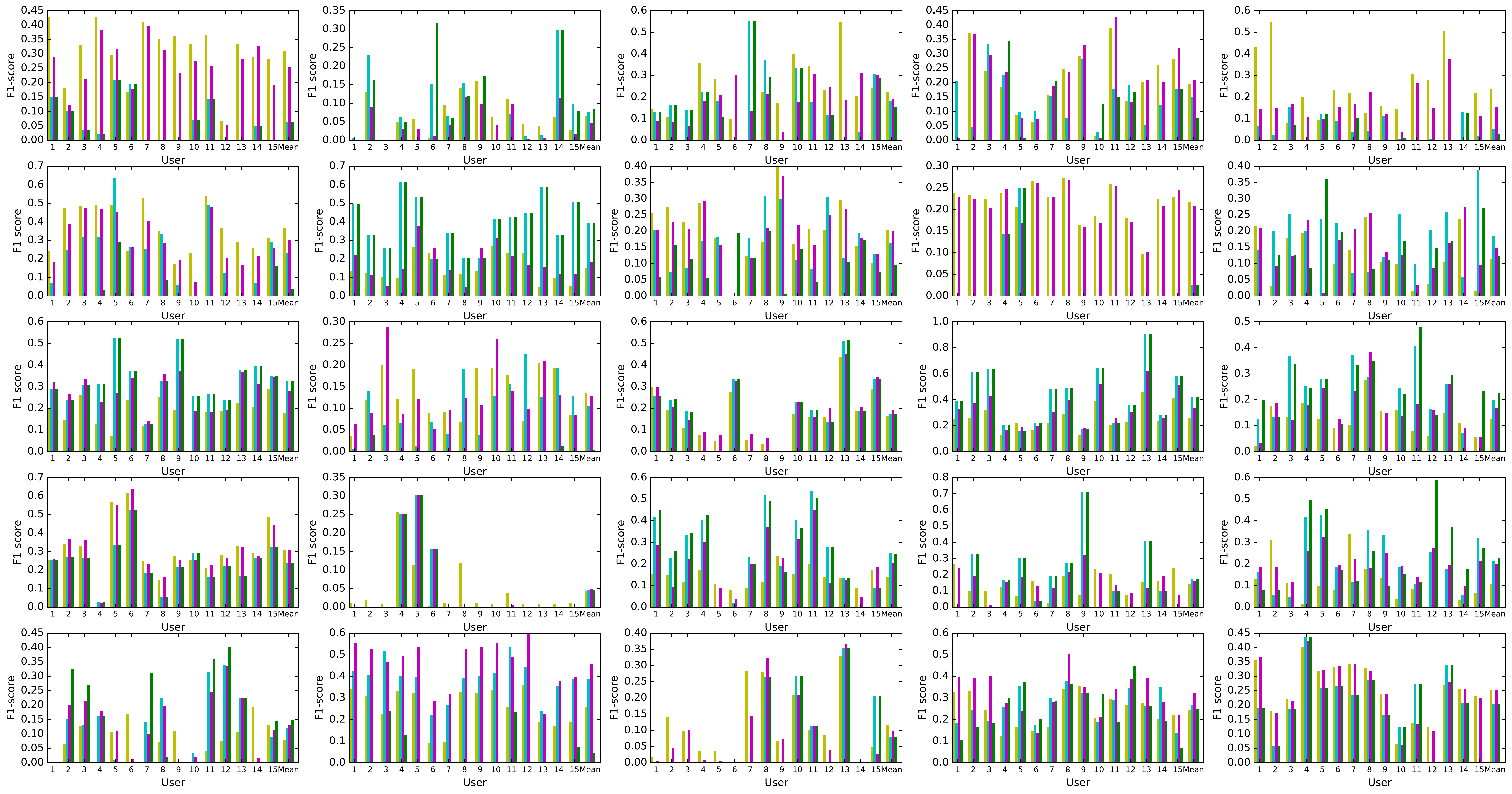}
	\end{center}
	\caption{F1-score of the summaries generated by greedy (yellow bar), sieve-streaming (cyan bar), stream clipper (magenta bar) and the first $15\%$ frames (green bar) comparing to reference summaries from $15$ users on $25$ videos from SumMe dataset. Each plot associates with a video. Stream clipper performs similar to or better than lazy greedy in most plots. }
	\label{fig:videoF1_all}
\end{figure*}

We also observe similar result on DUC 2001 corpus, which are composed of two datasets. The first one includes $60$ sets of documents, each is selected by a NIST assessor because the documents in a set are related to a same topic. The assessor also provides four human generated summary having word counts $400,200,100,50$ for each set. In Figure~\ref{fig:DUC2001_400_utility} and Figure~\ref{fig:DUC2001_200_utility}, we report the statistics to rouge-2 and F1-score of summaries of the same size generated by different algorithms. The second dataset is composed of four document sets associated with four topics. We report the detailed results in Table \ref{table:DUC01}. Both of them show stream clipper can achieve similar performance as offline greedy algorithm, whereas outperforms sieve-streaming.

\subsection{Video Summarization}

\begin{figure*}[t]
	\vspace{-0.5em}
	\begin{center}
		\includegraphics[width=1\linewidth]{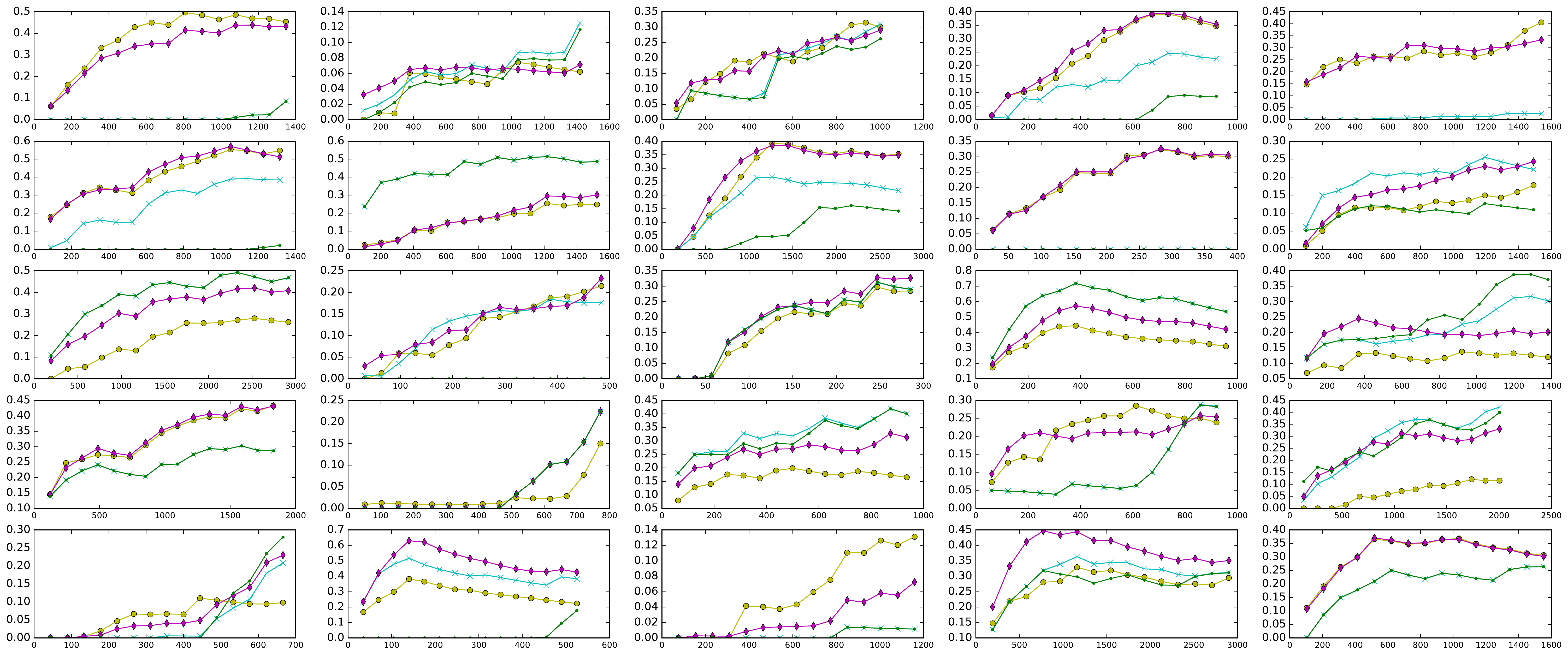}
	\end{center}
	\vspace{-0.8em}
	\caption{F1-score of the summaries generated by lazy greedy (yellow ``{\color{yellow}{$\bullet$}}''), sieve-streaming (cyan ``{\color{cyan}{$\times$}}''), stream clipper (magenta ``{\color{magenta}{$\blacklozenge$}}'') and the first $15\%$ frames (green ``{\color{green}{$\cdot$}}'') comparing to reference summaries of different sizes between $[0.02|V|, 0.32|V|]$ based on ground truth score (voting from $15$ users) on $25$ videos from SumMe. Each plot associates with a video. Stream clipper performs similar to or better than lazy greedy in most plots where sieve-streaming peforms poorly. In the plots where sieve-streaming outperforms others, its performance usually overlaps with that of the first $15\%$ frames. This is consistent with our observation in experiments that sieve-streaming usually saturates the solution $S$ by the first several frames and thus results in a trivial solution $S$.}
	\label{fig:videoF1}
	\vspace{-0.5em}
\end{figure*}

We apply lazy greedy, sieve-streaming, and stream clipper to $25$ videos from video summarization dataset SumMe \cite{Summe}\footnote{\texttt{http://www.vision.ee.ethz.ch/$\sim$gyglim/vsum/}}. Each video has $1000\sim 10000$ frames as given in Table \ref{table:SumMe} \cite{Supp}. We resize each frame to a $180\times 360$ image, and extract features from two standard image descriptors, i.e., a pyramid of HoG (pHoG) \cite{pHoG} to delineate local and global shape, and GIST \cite{GIST} to capture global scene. The $2728$ pHoG features are achieved over a four-level pyramid using $8$ bins with angle of $360$ degrees. The $256$ GIST features are obtained by using $4\times 4$ blocks and $8$ orientation per scale. We concatenate them to form a $2984$-dimensional feature vector for each frame to build $f(\cdot)$. Each algorithm selects $15\%$ of all frames as summary set, i.e., $k=0.15|V|$. Sieve-streaming uses a memory of $10k$ frames, while stream clipper uses a much smaller memory of $300+k$ frames. 

We compare the summaries generated by the three algorithms with the ones produced by the ground truth and $15$ users. Each user was asked to select a subset of frames as summary, and ground truth score of each frame is given by voting from all $15$ users. For each video, we compare each algorithm generated summary with the reference summary composed of the top $p$ frames with the largest ground truth scores for different $p$, and the user summary from different users. In particular, we report F1-score for comparison to ground truth score generated summaries in Figure~\ref{fig:videoF1} (recall comparison is given in Figure~\ref{fig:videoR} \cite{Supp}). We report F1-score for comparison to user summaries in Figure~\ref{fig:videoF1_all} (recall comparison is given in Figure~\ref{fig:videoR_all} \cite{Supp}). In each plot for each video, we also report the average F1-score and average recall over all $15$ users.

Stream clipper approaches or outperforms lazy greedy and shows high F1-score on most videos, while the time cost is small according to Table \ref{table:SumMe}. Although on a few videos sieve-streaming achieves the best F1-score, in most of these cases its generated summaries are trivially dominated by the first $15\%$ frames as shown in Figure~\ref{fig:videoF1_all}-\ref{fig:videoR_all} \cite{Supp}. On these videos, neither lazy greedy nor stream clipper performs well, though they acheive high objective value in optimization. This indicates that the extracted features of the submodular function should be improved.


\section{Conclusion}

In this paper, we introduce stream clipper, a fast and memory-efficient streaming submodular maximization algorithm that can achieve similar performance as commonly used greedy algorithm. It uses two thresholds to either select important element into summary or a buffer. The final summary is generated by greedily selecting more elements from the buffer. Swapping and buffer-reduce procedures are triggered lazily for further improvement and bounding memory. Thresholds are adjusted adaptively to avoid search for the optimal thresholds. 

\bibliographystyle{ACM-Reference-Format}
\bibliography{SC}

\newpage
\appendix


\section{Proof of Theorem \ref{the:k1Snbound}}

\begin{proof}
	We use $j\in[n+1,n+k-|S_n|]$ to index the $(j-n)^{th}$ step of the greedy algorithm in A\ref{alg:StreamBuffer}.L8-10, while $j=n$ indexes variables after passing $n$ elements and before the greedy procedure in A\ref{alg:StreamBuffer}.L8-10. Note $j=n+k-|S_n|$ indexes the final step of the greedy procedure.
	We have
	\begin{equation}\label{equ:t21}
	\begin{array}{ll}
	&f(S^*)\\
	&\leq f(S_j\cup S^*)\leq f(S_j)+\sum\limits_{v\in S^*\backslash S_j}f(v|S_j)\\
	&\leq f(S_j)+\sum\limits_{v\in S^*\backslash (S_j\cup B_n)}f(v|S_j)+\sum\limits_{v\in S^*\cap B_n}f(v|S_j)\\
	&\leq f(S_j)+k_n\tau^-+(k-k_n)[f(S_{j+1})-f(S_j)].
	\end{array}
	\end{equation}
	The first inequality uses monotonicity of $f(\cdot)$, while the second one is due to submodularity. 
	
	The third inequalities follows from set theory along with
	the fact that $f$ is non-negative monotone non-decreasing.
	The fourth inequality is a result of applying rejection rule $f(v|S)< \tau^-$ to the $k_n$ rejected elements in $S^*\backslash (S_n\cup B_n)$, and the max greedy selection rule in A\ref{alg:StreamBuffer}.L9. Rearranging (\ref{equ:t21}) yields
	\begin{equation}
	[f(S^*)-k_n\tau^-]-f(S_j)\leq (k-k_n)[f(S_{j+1})-f(S_j)]
	\end{equation}
	Let
	\begin{equation}
	\delta_j=[f(S^*)-k_n\tau^-]-f(S_j),
	\end{equation}
	then the rearranged inequality equals to
	\begin{equation}\label{equ:deltaj}
	\delta_j\leq (k-k_n)[\delta_j-\delta_{j+1}],
	\end{equation}
	When $\delta_j-\delta_{j+1} > 0$ and $k-k_n>0$, 
	this is exactly
	\begin{equation}\label{equ:jinequ}
	\delta_{j+1}\leq\left(1-\frac{1}{k-k_n}\right)\delta_j.
	\end{equation}
	Since in total $k-|S_n|$ elements are selected by the greedy algorithm, applying (\ref{equ:jinequ}) from $j=n+k-|S_n|-1$ to $j=n$ yields
	\begin{equation}
	\delta_{n+k-|S_n|}\leq \left(1-\frac{1}{k-k_n}\right)^{k-|S_n|}\delta_n\leq e^{-\frac{k-|S_n|}{k-k_n}}\delta_n,
	\end{equation}
	which is equivalent to
	\begin{equation}\label{equ:convex_comb}
	\begin{array}{ll}
	&f(S_{sc})\\
	&\geq \left(1-e^{-\frac{k-|S_n|}{k-k_n}}\right)[f(S^*)-k_n\tau^-]+e^{-\frac{k-|S_n|}{k-k_n}}f(S_n)\\
	&\geq \left(1-e^{-\frac{k-|S_n|}{k-k_n}}\right)[f(S^*)-k_n\tau^-]+e^{-\frac{k-|S_n|}{k-k_n}}|S_n|\tau^+
	\end{array}
	\end{equation}
	by applying the definition of $\delta_j$. The last inequality is due to 
	\begin{equation}\label{equ:large}
	f(S_n)=\sum_{v_i\in S_n}f(v_i|S_{i-1})\geq |S_n|\tau^+,
	\end{equation}
	which is due to selection rule $f(v|S)\geq \tau^+$ used in A\ref{alg:StreamBuffer}.L2. For each selected element $v_i\in S_n$, $S_{i-1}$ in (\ref{equ:large}) is the solution $S$ at the beginning of the $i^{th}$ step. We simply use telescope sum representation of $f(S_n)$ to achieve the equality in (\ref{equ:large}).
	
	When $\delta_j-\delta_{j+1} = 0$, or $k-k_n=0$, or both are zeros, (\ref{equ:deltaj}) implies $\delta_j\leq 0$, which leads to
	\begin{equation}\label{equ:zero_case}
	f(S_{sc})\geq f(S_j)\geq f(S^*)-k_n\tau^-.
	\end{equation}
	Note the right hand side of (\ref{equ:convex_comb}) is a convex combination of $f(S^*)-k_n\tau^-$ and $|S_n|\tau^+$, and thus is smaller than or equal to their maximum. If $f(S^*)-k_n\tau^-\geq|S_n|\tau^+$, (\ref{equ:zero_case}) directly leads to (\ref{equ:convex_comb}). If $f(S^*)-k_n\tau^-<|S_n|\tau^+$, since (\ref{equ:large}) is true and $f(S_{sc})\geq f(S_n)$ (because greedy algorithm cannot decrease $f(\cdot)$), (\ref{equ:convex_comb}) still holds.
	
	This completes the proof.
\end{proof}

\section{Proof of Proposition \ref{prop:boundmin}}

\begin{proof}
	In discussion of the minimum of $g(k_n)$, we frequently use its derivative and second order derivative w.r.t. $k_n$. For simplicity, we use $g^{'}(k_n)$ to denote $\frac{\partial g}{\partial k_n}$ and $g^{''}(k_n)$ to denote $\frac{\partial^2 g}{\partial k_n^2}$. 
	
	The second order derivative in (\ref{equ:2derivative}) can be represented as the product of a positive function $u(k_n)$ and a linear function $v(k_n)$ of $k_n$, i.e.,
	\begin{equation}\label{equ:2ndg}
	g^{''}(k_n)=u(k_n)\cdot v(k_n),
	\end{equation}
	where 
	\begin{equation}
	u(k_n)=e^{-\frac{k-|S_n|}{k-k_n}}\frac{k-|S_n|}{(k-k_n)^2},
	\end{equation}
	and
	\begin{equation}
	\begin{array}{ll}
	v(k_n)=&(k-2k_n+|S_n|)(f(S^*)-|S_n|\tau^+)-\\
	&(2k^2-3kk_n+|S_n|k_n)\tau^-.
	\end{array}
	\end{equation}
	In order to study the monotonicity of $g^{'}(k_n)$ in $k_n\in[0,k]$, we have to study the sign of $g^{''}(k_n)$ given in (\ref{equ:2ndg}). Since $u(k_n)$ is always positive, the sign of $g^{''}(k_n)$ depends on the sign of $v(k_n)$.
	
	The straight line $v(k_n)$ starts from $[0,v(0)]$, and ends at $[k,v(k)]$ with
	\begin{equation}\label{equ:vk_neg}
	v(k)=(|S_n|-k)\left[f(S^*)-|S_n|\tau^+-k\tau^-\right]\leq 0.
	\end{equation}
	Since we already know $v(k_n)$ is linear and monotone and $v(k)\leq 0$, how the sign of $v(k_n)$ changes in $k_n\in[0,k]$ depends on the sign of $v(0)$. The following discusses two cases when $v(0)\leq 0$ and $v(0)\geq 0$.
	\begin{enumerate}
		\item When $v(0)\leq 0$, $v(k_n)$ is non-positive due to its monotonicity of linear function, that is, $v(k_n)\leq 0$ for all $k_n\in[0, k]$. Combining with the fact $u(k_n)\geq 0$, we have $g^{''}(k_n)\leq 0$. Hence $g^{'}(k_n)$ is decreasing monotonically in $k_n\in[0,k]$. Because
		\begin{equation}\label{equ:ggk_neg}
		g^{'}(k)=-\tau^-\leq 0,
		\end{equation}
		we now only need to discuss two cases, $g^{'}(0)\geq 0$ and $g^{'}(0)\leq 0$.
		\begin{enumerate}
			\item When $g^{'}(0)\geq 0$, since $g^{'}(k_n)$ is monotonically decreasing, $g^{'}(k_n)$ starts from a nonnegative value $g^{'}(0)$ at $k_n=0$, passes $g^{'}(k_n)=0$ and keeps negative till $k_n=k$. Hence, $g(k_n)$ firstly increases and then decreases in $k_n\in [0, k]$. Therefore, $\min_{k_n\in[0,k]}g(k_n)$ has to be either $g(0)$ or $g(k)$.
			\item When $g^{'}(0)\leq 0$, since $g^{'}(k_n)$ is decreasing monotonically and $g^{'}(k)\leq 0$, $g^{'}(k_n)\leq 0$ in $k_n\in[0, k]$. Hence, $g(k_n)$ monotonically decreases in $k_n\in[0, k]$. Therefore, $\min_{k_n\in[0,k]}g(k_n)=g(k)$.
		\end{enumerate}
		\item When $v(0)\geq 0$, i.e.,
		\begin{equation}\label{equ:v0neg}
		f(S^*)-|S_n|\tau^+\geq\frac{2k^2}{k+|S_n|}\cdot \tau^-,
		\end{equation}
		since $v(k_n)$ is linear and $v(k)\leq 0$ in (\ref{equ:vk_neg}), it monotonically decreases in $k_n\in[0, k]$. Because $u(k_n)$ is positive, $g^{''}(k_n)$ starts from a nonnegative value $g^{''}(0)$, passes zero and ends at a non-positive value $g^{''}(k)$. This implies $g^{'}(k_n)$ firstly increases and then decreases in $k_n\in[0,k]$. In order to decide the shape of $g(k_n)$, we further need the sign of $g^{'}(0)$ and $g^{'}(k)$.
		
		We already know $g^{'}(k)\leq 0$ from (\ref{equ:ggk_neg}). So we only need to determine whether $g^{'}(0)\geq 0$ or not. According to the derivative given in (\ref{equ:derivative}),
		\begin{equation}\label{equ:gpartial0}
		\begin{array}{ll}
		g^{'}(0)=&e^{-\frac{k-|S_n|}{k}}\left[\left(1-e^{\frac{k-|S_n|}{k}}\right)\tau^-+\right.\\
		&\left.\frac{k-|S_n|}{k^2}\left(f(S^*)-|S_n|\tau^+\right)\right].
		\end{array}
		\end{equation}
		Applying the inequality (\ref{equ:v0neg}) to the second term of the right hand side in above equation yields
		\begin{equation}\label{equ:part2}
		\begin{array}{ll}
		&\left(1-e^{\frac{k-|S_n|}{k}}\right)\tau^-+\frac{k-|S_n|}{k^2}\left(f(S^*)-|S_n|\tau^+\right)\geq\\
		&~~~~~~~~~~\left[1-e^{\frac{k-|S_n|}{k}}+\frac{2(k-|S_n|)}{k+|S_n|}\right]\tau^-.
		\end{array}
		\end{equation}
		Let
		\begin{equation}
		h(|S_n|)=1-e^{\frac{k-|S_n|}{k}}+\frac{2(k-|S_n|)}{k+|S_n|},
		\end{equation}
		then
		\begin{equation}
		h^{'}(|S_n|)=e^{\frac{k-|S_n|}{k}}-\frac{4k^2}{(k+|S_n|)^2}.
		\end{equation}
		It is not hard to verify both $e^{\frac{k-|S_n|}{k}}$ and $\frac{4k^2}{(k+|S_n|)^2}$ decreases monotonically w.r.t. $|S_n|$. In addition, the former is smaller than the latter when $|S_n|=0$ and equal to the latter when $|S_n|=k$, so $\frac{4k^2}{(k+|S_n|)^2}\geq e^{\frac{k-|S_n|}{k}}$ in $|S_n|\in[0, k]$. Hence, $h^{'}(|S_n|)\leq 0$ and thus $h(|S_n|)$ monotonically decreases. Therefore,
		\begin{equation}
		h(|S_n|)\geq h(k)=0.
		\end{equation}
		Recall (\ref{equ:part2}) and (\ref{equ:gpartial0}), we have $g^{'}(0)\geq 0$. So $g^{'}(k_n)$ starts from a nonnegative value at $k_n=0$, firstly increases and then decreases to a non-positive value at $k_n=k$. Hence, $g(k_n)$ increases at first and then decreases in $k_n\in[0, k]$. Therefore, $\min_{k_n\in[0,k]}g(k_n)$ has to be either $g(0)$ or $g(k)$.
	\end{enumerate}
	This completes the discussion of all possible cases and thus finishes the proof.
\end{proof}

\section{Proof of Theorem \ref{the:Snbound}}

\begin{proof}
	We prove the lower bound in three difference cases given in Theorem \ref{the:Snbound}. Note the bound given in (\ref{equ:k1Snbound}) is a convex combination of $f(S^*)-k_n\tau^-$ and $|S_n|\tau^+$. In the first case, we do not use (\ref{equ:k1Snbound}) to derive the bound because the condition for the first case equals to $f(S^*)-k_n\tau^-\leq |S_n|\tau^+$. So the largest value (\ref{equ:k1Snbound}) can achieve is $|S_n|\tau^+$. But according to (\ref{equ:large}), we have $f(S_n)\geq |S_n|\tau^+$. 
	
	Since greedy algorithm always improves the objective $f(S_{sc})$ i.e., $f(S_{sc})\geq f(S_n)$, we have 
	\begin{equation}\label{equ:fSlarge}
	f(S_{sc})\geq f(S_n)\geq |S_n|\tau^+
	\end{equation}
	In addition, $k_n\leq k$ and $|S_n|\leq k$ lead to
	\begin{equation}
	f(S^*)\leq k_n\tau^-+|S_n|\tau^+\leq k(\tau^-+\tau^+)\leq 2k\tau^+,
	\end{equation}
	which indicates $\tau^+\geq \frac{f(S^*)}{2k}$.
	Combining with (\ref{equ:fSlarge}), we have
	\begin{equation}
	f(S_{sc})\geq |S_n|\tau^+\geq \frac{|S_n|}{k}\times \frac{f(S^*)}{2}.
	\end{equation}
	
	The proof for the other two cases relies on Proposition \ref{prop:boundmin}. According to it, in the second and third case when $f(S^*)\geq k_n\tau^-+|S_n|\tau^+$, the minimum of the lower bound in (\ref{equ:k1Snbound}) w.r.t. $k_n\in[0,1]$ is either $g(k)$ or $g(0)$ given in Proposition \ref{prop:boundmin}. So we only need to find out which one is smaller in each case since we are searching for the worst bound w.r.t. $k_n$.
	
	According to Theorem \ref{the:Snbound}, the condition for the second case is
	\begin{equation}\label{equ:cond2}
	f(S^*)\leq e^{1-\frac{|S_n|}{k}}k\tau^-+|S_n|\tau^+,
	\end{equation}
	which is a rearrangement of $g(k)\leq g(0)$, i.e.,
	\begin{equation}\label{equ:cond2a}
	f(S^*)-k\tau^-\leq \left(1-e^{-1+\frac{|S_n|}{k}}\right)f(S^*)+e^{-1+\frac{|S_n|}{k}}|S_n|\tau^+.
	\end{equation}
	Therefore, the lower bound in the second case is $f(S^*)-k\tau^-$.
	
	In the third case, by reversing the inequalities in both (\ref{equ:cond2}) and (\ref{equ:cond2a}), we can prove $g(0)\leq g(k)$ and thus the lower bound of the third case is $g(0)$. This completes the proof.
\end{proof}

\section{Proof of Corollary \ref{cor:bestbound}}

\begin{proof}
	In the first case, according to (\ref{equ:bound1}), by using $\tau^-\leq \frac{f(S^*)}{2k}$, we have
	\begin{align}
	\notag f(S)&\geq |S_n|\tau^+\geq f(S^*)-k_n\tau^-\\
	&\geq \left(1-\frac{k_n}{2k}\right)f(S^*)\geq\frac{f(S^*)}{2}.
	\end{align}
	In the second case, the condition is equivalent to
	\begin{equation}
	\frac{f(S^*)-|S_n|\tau^+}{ke^{1-\frac{|S_n|}{k}}}\leq \tau^-\leq \frac{f(S^*)-|S_n|\tau^+}{k_n}.
	\end{equation}
	Since the lower bound $f(S^*)-k_n\tau^-$ (\ref{equ:bound2}) increases when reducing $\tau^-$, its maximum w.r.t. $\tau^-$ is achieved when $\tau^-=(f(S^*)-|S_n|\tau^+)/(ke^{1-\frac{|S_n|}{k}})$., i.e.,
	\begin{equation}
	f(S^*)-k\tau^-=\left(1-e^{-1+\frac{|S_n|}{k}}\right)f(S^*)+e^{-1+\frac{|S_n|}{k}}|S_n|\tau^+,
	\end{equation} 
	which is exactly the same bound (\ref{equ:bound3}) for the third case.
	
	In the third case, by substituting $|S_n|=0$ into (\ref{equ:bound3}), the bound in the third case becomes
	\begin{equation}
	f(S_{sc})\geq \left(1-e^{-1}\right)f(S^*).
	\end{equation}
	This completes the proof.
\end{proof}

\section{Proof of Proposition \ref{the:combine23}}

The following corollary is derived from the bound of Case 1 in Theorem \ref{the:Snbound}. 
\begin{corollary}\label{cor:case1}
	For $\alpha\in[0,k_n/k]$, given $\tau^-$ and $\tau^+\geq \tau^-$, if
	\begin{align}\label{equ:case1t}
	\notag &m^*\triangleq\min\left\{m\in[n]:\tau^+\geq\frac{(1-\alpha)f(S^*)}{m}\right\},\\
	&~~~~~~~~~~~\frac{\alpha f(S^*)}{k_n}\leq \tau^-\leq \tau^+,
	\end{align}
	for any order $\sigma\in\Phi_1(\alpha)$ where
	\begin{equation}\label{equ:case12}
	\Phi_1(\alpha)\triangleq\left\{\sigma\in\Sigma: |S_n|\geq m^*, |B_n|\leq b\right\},
	\end{equation}
	we have
	\begin{equation}
	f(S)\geq(1-\alpha) f(S^*).
	\end{equation}
\end{corollary}
\begin{proof}
	The inequalities about $\tau^-$ and $\tau^+$ in (\ref{equ:case1t}) lead to
	\begin{align}
	\notag k_n\tau^-&\geq \alpha f(S^*)=f(S^*)-(1-\alpha)f(S^*)\\
	&\geq f(S^*)-m^*\tau^+\geq f(S^*)-|S_n|\tau^+,
	\end{align}
	which after rearrangement is the condition for Case 1 in Theorem \ref{the:Snbound}. Substitute the inequality about $\tau^+$ into the bound (\ref{equ:bound1}) for Case 1, we have
	\begin{equation}
	f(S_{sc})\geq(1-\alpha) f(S^*).
	\end{equation}
	Our assumption requires $|B_n|\geq k-|S_n|$, which requires $\tau^-\leq f(S^*)/k$ because otherwise $|B_n|+|S_n|\leq k$. So the lower bound of $\tau^-$ in (\ref{equ:case1t}) needs to satisfy
	\begin{equation}
	\frac{\alpha f(S^*)}{k_n}\leq \tau^-\leq \frac{f(S^*)}{k},
	\end{equation}
	which equals to $\alpha\in[0,k_n/k]$. In addition, the buffer size limit requires $|B_i|\leq b~ \forall i\in[n]$. Since $|B_i|\leq |B_{i+1}|$, we require $|B_n|\leq b$. This complete the proof.
\end{proof}

The following corollary is derived from the bound of Case 2 in Theorem \ref{the:Snbound}.
\begin{corollary}\label{cor:case2}
	For $\alpha\in[0,1/2]$, given $\tau^-$ and $\tau^+\geq\tau^-$, if
	\begin{equation}\label{equ:case2t}
	\begin{array}{ll}
	&\tau^-\leq \frac{\alpha f(S^*)}{k},\\ &M\triangleq\left\{m\in[n]:\frac{f(S^*)-e^{1-\frac{m}{k}}k\tau^-}{m}<\tau^+<\frac{f(S^*)-k\tau^-}{m}\right\},
	\end{array} 
	\end{equation} 
	for any order $\sigma\in\Phi_2(\alpha)\triangleq\bigcup_{m\in M}\Psi_2(\alpha, m)$ where
	\begin{equation}\label{equ:case21}
	\Psi_2(\alpha, m)\triangleq\left\{\sigma\in\Sigma: |S_n|=m, |B_n|\leq b\right\},
	\end{equation}
	we have
	\begin{equation}
	f(S_{sc})\geq(1-\alpha)f(S^*).
	\end{equation}
\end{corollary}
\begin{proof}
	Since $|S_n|=m$, rearranging the inequality about $\tau^+$ in (\ref{equ:case2t}) leads to the condition for Case 2 in Theorem \ref{the:Snbound}, i.e.,
	\begin{equation}
	k_n\tau^-+|S_n|\tau^+< f(S^*)< e^{1-\frac{|S_n|}{k}}k\tau^-+|S_n|\tau^+.
	\end{equation}
	Substituting the inequality about $\tau^-$ in (\ref{equ:case2t}) into the bound (\ref{equ:bound2}) for Case 2 results in
	\begin{equation}
	f(S)\geq(1-\alpha)f(S^*).
	\end{equation}
	The above holds for all $m\in M$. The buffer size limit requires $|B_n|\leq b$. This completes the proof.
\end{proof}


The following corollary is derived from the bound of Case 3 in Theorem \ref{the:Snbound}.
\begin{corollary}\label{cor:case3}
	1) When $\alpha\in[0,1/e]$, given $\tau^-$ and $\tau^+\geq\tau^-$, if 
	\begin{equation}\label{equ:case3t1}
	\begin{array}{ll}
	m^*\triangleq\min&\left\{m\in[n]:\tau^+\geq\frac{(1-e^{1-m/k}\alpha)f(S^*)}{m},\right.\\
	&\left.\tau^-\leq \frac{f(S^*)-m\tau^+}{e^{1-m/k}k}\right\},
	\end{array}
	\end{equation}
	for any order $\sigma\in\Phi_3(\alpha)$ where 
	\begin{equation}\label{equ:case31}
	\Phi_3(\alpha)\triangleq\left\{\sigma\in\Sigma: |S_n|\geq m^*, |B_n|\leq b\right\},
	\end{equation}
	we have
	\begin{equation}
	f(S_{sc})\geq(1-\alpha)f(S^*).
	\end{equation}
	2) When $\alpha\in(1/e,1/2]$, given $\tau^-$ and $\tau^+\geq\tau^-$, if 
	\begin{equation}\label{equ:case3t3}
	\begin{array}{ll}
	&M_1\triangleq\left\{m\in[n]:\frac{(1-\alpha)f(S^*)}{k}\leq \tau^+\leq\frac{f(S^*)}{m+k}\right\}, \\
	&M_2\triangleq\left\{m\in[n]:\tau^-\leq \frac{f(S^*)-m\tau^+}{e^{1-m/k}k}\right\},
	\end{array}
	\end{equation}
	for any order $\sigma\in\Phi_3(\alpha)\triangleq\bigcup_{m\in M_1\cap M_2}\Psi_3(\alpha, m)$ where 
	\begin{equation}\label{equ:case33}
	\Psi_3(\alpha, m)\triangleq\left\{\sigma\in\Sigma: |S_n|=m, |B_n|\leq b\right\},
	\end{equation}
	we have
	\begin{equation}
	f(S_{sc})\geq(1-\alpha)f(S^*),
	\end{equation}
\end{corollary}
\begin{proof}
	When 1) $\alpha\in[0,1/e]$, because the increasing monotonicity
	\begin{align}\label{equ:de1}
	\notag &\partial\left(\frac{f(S^*)-|S_n|\tau^+}{e^{1-|S_n|/k}k}\right)/\partial |S_n|\\
	=&\frac{(k-|S_n|)\tau^++f(S^*)}{k^2e^{1-|S_n|/k}}>0,
	\end{align}
	and $|S_n|\geq m^*$ in \ref{equ:case31}, the inequality about $\tau^-$ in (\ref{equ:case3t1}) leads to
	\begin{equation}
	\tau^-\leq \frac{f(S^*)-m^*\tau^+}{e^{1-m^*/k}k}\leq \frac{f(S^*)-|S_n|\tau^+}{e^{1-|S_n|/k}k},
	\end{equation}
	which after rearranging is the condition for Case 3 in Theorem \ref{the:Snbound}. Since $\alpha\leq1/e$, we have
	\begin{align}\label{equ:de2}
	\notag &\partial\left[\frac{(1-e^{1-|S_n|/k}\alpha)f(S^*)}{|S_n|}\right]/\partial |S_n|\\
	\notag =&\frac{f(S^*)\left[\alpha e^{1-|S_n|/k}(1+|S_n|/k)-1\right]}{|S_n|^2}\\
	\leq&\frac{f(S^*)\left[\alpha e-1\right]}{|S_n|^2}\leq 0,
	\end{align}
	where the first inequality is due to $1+x\leq e^x$. So for $|S_n|\geq m^*$, combining the above non-increasing monotonicity and the inequality about $\tau^+$ in (\ref{equ:case3t1}) yields
	\begin{equation}
	\tau^+\geq\frac{(1-e^{1-m^*/k}\alpha)f(S^*)}{m^*}\geq \frac{(1-e^{1-|S_n|/k}\alpha)f(S^*)}{|S_n|}.
	\end{equation}
	Substituting the above inequality into bound (\ref{equ:bound3}) for Case 3 results in
	\begin{equation}
	f(S_{sc})\geq(1-\alpha)f(S^*).
	\end{equation}
	
	
	When 2) $\alpha\in(1/e,1/2]$, for each $m\in M_1\cap M_2$, we have
	\begin{equation}
	\tau^-\leq \frac{f(S^*)-m\tau^+}{e^{1-m/k}k},
	\end{equation}
	which after rearranging is the condition for Case 3 in Theorem \ref{the:Snbound}. According to the inequality about $\tau^+$ in (\ref{equ:case3t3}), for each $m\in M_1\cap M_2$, we have
	\begin{equation}\label{equ:tplusrange}
	\frac{(1-\alpha)f(S^*)}{k}\leq \tau^+\leq\frac{f(S^*)}{m+k},
	\end{equation}
	where the right inequality indicates that the bound (\ref{equ:bound3}) for Case 3 is monotone decreasing w.r.t. $|S_n|$ because the derivative of the bound (\ref{equ:bound3}) w.r.t. $|S_n|$ is 
	\begin{equation}\label{equ:de3}
	\frac{e^{-1+|S_n|/k}}{k}\times\left[(m+k)\tau^+-f(S^*)\right]\leq 0.
	\end{equation}
	So the bound (\ref{equ:bound3}) for $|S_n|=m$ is larger than the bound when $|S_n|=k$, which is $k\tau^+$. By using the left inequality in (\ref{equ:tplusrange}), we have
	\begin{equation}
	f(S_{sc})\geq(1-\alpha)f(S^*).
	\end{equation}
	
	This completes the proof.
\end{proof}

\begin{lemma}\label{lemma:combine23}
	When $\alpha\in[0,1/e]$, $\Phi_2(\alpha)\subseteq \Phi_3(\alpha)$.
\end{lemma}
\begin{proof}
	When $\alpha\in[0,1/e]$, for any $m\in M$, if the given $\tau^-$ and $\tau^+\geq\tau^-$ fulfill the inequalities in (\ref{equ:case2t}), they also fulfill the inequalities in (\ref{equ:case3t1}), because 1)
	\begin{equation}
	\tau^-\leq\frac{\alpha f(S^*)}{k}\leq\frac{f(S^*)}{ek}\leq \frac{f(S^{*})-m^{*}\tau^+}{e^{1-m^*/k}k}
	\end{equation}
	the last inequality is due to the monotonicity according to the derivative given in (\ref{equ:de1}), so $\tau^-$ fulfills the inequality about $\tau^-$ in (\ref{equ:case3t1}); and 2)
	\begin{align}
	\notag\tau^+>&\frac{f(S^*)-e^{1-\frac{m}{k}}k\tau^-}{m}\\
	\notag \geq&\frac{(1-e^{1-m/k}\alpha)f(S^*)}{m}\\
	\geq&\frac{(1-e^{1-m^*/k}\alpha)f(S^*)}{m^*}, 
	\end{align}
	the second inequality is due to $\tau^-\leq \frac{\alpha f(S^*)}{k}$, the third inequality is due to the monotonicity according to the derivative given in (\ref{equ:de2}), so $\tau^+$ fulfills the inequality about $\tau^+$ in (\ref{equ:case3t1}). This completes the proof.
\end{proof}

Since $k_n$ in (\ref{equ:case1t}) cannot be known, we will not use Corollary \ref{cor:case1} to derive order complexity. Combing the results of Corollary \ref{cor:case2} and Corollary \ref{cor:case3} by using Lemma \ref{lemma:combine23} yields Proposition \ref{the:combine23}.

\label{sec:extension}
\section{Extending Analysis of Algorithm~\ref{alg:StreamBuffer} to Algorithm~\ref{alg:StreamClipper}}

Comparing to Algorithm~\ref{alg:StreamBuffer} that fixes $\tau^-$ and $\tau^+$ as constants, Algorithm~\ref{alg:StreamClipper} updates $\tau^-$ and $\tau^+$ within the two strategies (A\ref{alg:StreamClipper}.L8, L9, L16). The changes in thresholds may lead to difference in theoretical analysis. The following analysis shows how to extend the approximation bound and order complexity of Algorithm~\ref{alg:StreamBuffer} to Algorithm~\ref{alg:StreamClipper}.

Firstly, we study the reason for increasing $\tau^-$ in A\ref{alg:StreamClipper}.L8 of swapping procedure. The following Lemma indicates how marginal gain $f(w|S)$ changes after swapping. 
\begin{lemma}\label{prop:swap}
	If $f(\cdot)$ is a normalized submodular function, $u\in S$, $v,w\notin S$, then the following holds.
	\begin{equation}
	f(w|S\backslash u\cup v)\leq f(w|S)+f(u|S\backslash u\cup v).
	\end{equation}
\end{lemma}
\begin{proof}
	For simplicity in notations, in the following proof, we use ``$+$'' for set union operator ``$\cup$'' and ``$-$'' for set subtraction operator ``$\backslash$''. It can be proved as follows.
	\begin{equation}
	\begin{array}{ll}
	&f(w|S-u+v)\\
	=&f(w+S-u+v)-f(S+v)+\\
	&\left[f(S+v)-f(S-u+v)\right]\\
	=&f(w+S-u+v)-f(S+v)+f(u|S-u+v)\\
	=&f(w+S-u+v)-f(w+S+v)+\\
	&\left[f(w+S+v)-f(S+v)\right]+f(u|S-u+v)\\
	=&f(w|S+v)+f(u|S-u+v)-f(u|w+S-u+v)\\
	\leq& f(w|S)+f(u|S-b+v)-f(u|w+S-u+v)\\
	\leq& f(w|S)+f(u|S-u+v).
	\end{array}
	\end{equation}
	The first inequality is due to submodularity, and the second inequality is a result of nonnegativity.
\end{proof}

In the following, we use $\tau_i^-$ and $\tau_i^+$ to denote $\tau^-$ and $\tau^+$ at the end (i.e., at A\ref{alg:StreamClipper}.L17) of step $i$.

\begin{theorem}{\label{lemma:tt}}
	The approximation bound and order complexity of Algorithm~\ref{alg:StreamBuffer} holds true for Algorithm~\ref{alg:StreamClipper}, if $\tau^-$ and $\tau^+$ in Section \ref{sec:bound} are replaced respectively by $\tau_n^-$ and $\tau_n^+$.
\end{theorem}
\begin{proof}
	For Algorithm~\ref{alg:StreamClipper}, after replacing $\tau^-$ and $\tau^+$ by $\tau_{n}^-$ and $\tau_{n}^+$ respectively, the analysis in Section \ref{sec:bound} holds true if
	\begin{equation}\label{equ:condt}
	\begin{array}{ll}
	&\forall v\in V\backslash(S_n\cup B_n), f(v|S_n)\leq\tau_{n}^-,\\
	&f(S_n)\geq |S_n|\tau_{n}^+,
	\end{array}
	\end{equation} 
	because the first condition in (\ref{equ:condt}) leads to the third inequality in (\ref{equ:t21}) due to submodularity ($S_n\subseteq S_j$ where $j$ indexes the greedy steps after passing $n$ elements), while the second condition results in (\ref{equ:large}), where $\tau^-$ and $\tau^+$ are replaced by $\tau_{n}^-$ and $\tau_{n}^+$. The rest reasoning follows the proof of Theorem \ref{the:k1Snbound} and Theorem \ref{the:Snbound}, and lead to the same approximation bound and order complexity.
	
	In the following, we prove (\ref{equ:condt}) is true. Firstly, if $v_i$ is rejected in step $i$, and the first swapping since step $i$ happens at step $j>i$, i.e., some element $u\in S_{j-1}$ is replaced by $v_j$, the marginal gain $f(v_i|S_j)$ can be upper bounded by using Lemma \ref{prop:swap}, i.e., 
	\begin{equation}
	\begin{array}{ll}
	f(v_i|S_j)&\leq f(v_i|S_{j-1})+f(u|S_j)\leq f(v_i|S_i)+f(u|S_j)\\
	&\leq\tau_{i-1}^-+f(u|S_j)\leq\tau_{j-1}^-+f(u|S_j)\leq\tau_{j}^-\leq\tau_{n}^-.
	\end{array}
	\end{equation}
	The first inequality is due to $S_j=S_{j-1}\backslash u\cup v_j$ and Lemma \ref{prop:swap}, the second inequality is due to $S_i\subseteq S_{j-1}$ and submodularity, the third inequality is due to $f(v_i|S_i)\leq\tau_{i-1}^-$, the fourth inequality are due to the fact that $\tau^-$ is non-decreasing in Algorithm~\ref{alg:StreamClipper}, the fifth inequality is due to A\ref{alg:StreamClipper}.L8 and non-decreasing property of $\tau^-$. By induction, we have $f(v_i|S_n)\leq \tau_{n}^-$. If no swapping happens, for each rejected element $v_i$, we directly have $f(v_i|S_n)\leq f(v_i|S_i)\leq \tau_{i-1}^-\leq \tau_{n}^-$ because of submodularity. Therefore, the first condition in (\ref{equ:condt}) is true.
	
	Secondly, $f(S_{i-1})\geq |S_{i-1}|\tau_{i-1}^+$ holds for $i=1$ because $S_0=\emptyset$, assume $f(S_{i-1})\geq |S_{i-1}|\tau_{i-1}^+$, if swapping happens in step $i$, according to $\rho$ in A\ref{alg:StreamClipper}.L6 and A\ref{alg:StreamClipper}.L9, we have $f(S_i)\geq |S_i|(\tau_{i-1}^++\rho)=|S_i|\tau_{i}^+$; if no swapping happens, we have $f(S_i)\geq |S_i|\tau_{i-1}^+=|S_i|\tau_{i}^+$ because $\tau_{i-1}^+=\tau_{i}^+$. By induction, we have $f(S_n)\geq |S_n|\tau_n^+$. Therefore, the second condition in (\ref{equ:condt}) is true. This completes the proof.
\end{proof}

\section{Extensions to Other Constraints}

\subsection{Knapsack Constraint}

The problem is modified to
\begin{equation}\label{equ:knapsackp}
\max_{S\subseteq V} f(S)~~s.t.~~\sum_{v\in S}c(v)\leq b.
\end{equation}

The following modification needs to be applied to Algorithm~\ref{alg:StreamBuffer}. The thresholding of $f(v|S)$ changes to thresholding of $f(v|S)/c(v)$, i.e., it adds $v$ to $S$ if $\frac{f(v|S)}{c(v)}\geq \tau^+$, removes $v$ if $\frac{f(v|S)}{c(v)}\leq \tau^-$, and saves $v$ in buffer $B$ otherwise. Accordingly, A\ref{alg:StreamBuffer}.L9 in greedy stage is replaced by $v^*=\argmax_{v\in B} f(v|S)/c(v)$. Then we can achieve the following bound analogous to Theorem \ref{the:k1Snbound}.

\begin{theorem}
	After applying the above modification for knapsack constraint to Algorithm~\ref{alg:StreamBuffer}, the following holds for the output $\tilde S$ and is and the optimal set $S^*$ of problem (\ref{equ:knapsackp}).
	\begin{equation}\label{equ:knapsackb}
	f(\tilde S)\geq\frac{1}{2}\left[\left(1-e^{-\frac{b-\alpha}{b-\beta}}\right)\left(f(S^*)-\beta \tau^-\right)+e^{-\frac{b-\alpha}{b-\beta}}\alpha \tau^+\right],
	\end{equation}
	where
	\begin{equation}\label{equ:alphabeta}
	\begin{array}{ll}
	&\alpha = b-\sum_{v\in S_+\backslash S_n}c(v)\leq \sum_{v\in S_n}c(v)\\
	&\beta = b - \sum_{u\in S^*\cap B_n}c(u) = \sum_{u\in S^*\backslash(S_n\cup B_n)}c(u).
	\end{array}
	\end{equation}
	Here $S_+=S\cup v_+$ is the solution set $S$ which firstly violates the knapsack constraint $\sum_{v\in S}c(v)\leq b$ because of adding $v_+$ in the final greedy stage of stream clipper.
\end{theorem}
It can be verified that the right hand side of (\ref{equ:knapsackb}) is larger than $(1/4)f(S^*)$.
\begin{proof}
	By following the proof of greedy algorithm for submodular maximization with knapsack constraint, we study the solution $S_+=S\cup v_+$ firstly violating the constraint $\sum_{v\in S}c(v)\leq b$ in the greedy stage. For the first step of the modified greedy algorithm,
	\begin{equation}\label{equ:t22}
	\begin{array}{ll}
	&f(S^*)\\
	&\leq f(S_n\cup S^*)\leq f(S_n)+\sum\limits_{v\in S^*\backslash S_n}f(v|S_n)\\
	&\leq f(S_n)+\sum\limits_{v\in S^*\backslash (S_n\cup B_n)}c(v)\cdot\frac{f(v|S_n)}{c(v)}+\\
	&~~~~\sum\limits_{v\in S^*\cap B_n}c(v)\cdot\frac{f(v|S_n)}{c(v)}\\
	&\leq f(S_n)+\sum\limits_{u\in S^*\backslash(S_n\cup B_n)}c(u)\cdot \tau^-+\\
	&~~~~\sum\limits_{u\in S^*\cap B_n}c(u)\cdot\frac{f(S_{n+1})-f(S_n)}{c(S_{n+1}\backslash S_n)}\\
	&=f(S_n)+\beta \tau^-+\frac{b-\beta}{c(S_{n+1}\backslash S_n)}\left[f(S_{n+1})-f(S_n)\right].
	\end{array}
	\end{equation}
	After rearranging,
	\begin{equation}
	[f(S^*)-\beta \tau^-]-f(S_n)=\frac{b-\beta}{c(S_{n+1}\backslash S_n)}\left[f(S_{n+1})-f(S_n)\right]
	\end{equation}
	Let
	\begin{equation}
	\delta_n=f(S^*)-\beta \tau^-,
	\end{equation}
	then the rearranged inequality equals to
	\begin{equation}
	\delta_{n+1}\leq \left(1-\frac{c(S_{n+1}\backslash S_n)}{b-\beta}\right)\delta_n,
	\end{equation}
	It can be easily verified that the above inequality between $\delta_{n+1}$ and $\delta_n$ also holds for all $\delta_{i+1}$ and $\delta_i$ with $n\leq i\leq n+|S|-1$. Repeatedly applying these inequalities yields
	\begin{equation}
	\begin{array}{ll}
	\delta_{n+|S_+|-|S_n|}&\leq\prod\limits_{i=n}^{|S_+|-1}\left(1-\frac{c(S_{i+1}\backslash S_i)}{b-\beta}\right)\cdot\delta_n\\
	&=\prod\limits_{v\in S_+\backslash S_n}\left(1-\frac{c(v)}{b-\beta}\right)\cdot\delta_n\\
	&\leq\left(1-\frac{\sum_{v\in S_+\backslash S_n}c(v)}{(b-\beta)(|S_+|-|S_n|)}\right)^{|S_+|-|S_n|}\cdot\delta_n\\
	&=\left(1-\frac{b-\alpha}{(b-\beta)(|S_+|-|S_n|)}\right)^{|S_+|-|S_n|}\cdot\delta_n\\
	&\leq e^{-\frac{b-\alpha}{b-\beta}}\delta_n.
	\end{array}
	\end{equation}
	The second inequality is due to AM-GM inequality (i.e., the inequality of arithmetic and geometric means). Hence, we have
	\begin{equation}\label{equ:Splus}
	\begin{array}{ll}
	&f(S_+)\\
	&\geq \left(1-e^{-\frac{b-\alpha}{b-\beta}}\right)[f(S^*)-\beta \tau^-]+e^{-\frac{b-\alpha}{b-\beta}}f(S_n)\\
	&\geq \left(1-e^{-\frac{b-\alpha}{b-\beta}}\right)[f(S^*)-\beta \tau^-]+e^{-\frac{b-\alpha}{b-\beta}}\sum\limits_{v\in S_n}c(v)\tau^+\\
	&\geq \left(1-e^{-\frac{b-\alpha}{b-\beta}}\right)[f(S^*)-\beta \tau^-]+e^{-\frac{b-\alpha}{b-\beta}}\alpha \tau^+.
	\end{array}
	\end{equation}
	The second inequality in (\ref{equ:Splus}) is due to 
	\begin{equation}\label{equ:large1}
	f(S_n)\geq \sum_{v\in S_n}c(v)\tau^+,
	\end{equation}
	and the last inequality in (\ref{equ:Splus}) is due to the first inequality in (\ref{equ:alphabeta}), which is resulted from
	\begin{equation}
	b\leq \sum_{v\in S_+}c(v)\leq\sum_{v\in S_+\backslash S_n}c(v)+\sum_{v\in S_n}c(v).
	\end{equation}
	According to the modified greedy algorithm for knapsack constraint, given
	\begin{equation}
	x^*=\argmax_{v\in B_n}f(v),
	\end{equation}
	we have
	\begin{equation}
	f(S)+f(x^*)\geq f(S_+). 
	\end{equation}
	Therefore, the output $\tilde S=\argmax\{f(S),f(x^*)\}$ satisfies
	\begin{equation}
	f(\tilde S)\geq\frac{1}{2}f(S_+).
	\end{equation}
	Recall the lower bound of $f(S_+)$ given in (\ref{equ:Splus}), that completes the proof.
\end{proof}

\subsection{Matroid Constraint}

The problem is modified to
\begin{equation}\label{equ:matroidp}
\max_{S\subseteq V} f(S)~~s.t.~~S\in\mathcal I,
\end{equation}
where $\mathcal I$ is the set of independent sets.

The following modification needs to be applied to Algorithm~\ref{alg:StreamBuffer}. Extra condition $S\cup v_i\in\mathcal I$ is added to A\ref{alg:StreamBuffer}.L2. And A\ref{alg:StreamBuffer}.L10 is executed only when $S\cup v^*\in\mathcal I$. 

\begin{theorem}
	After the above modification, Algorithm~\ref{alg:StreamBuffer} outputs a solution $S$ with approximation bound
	\begin{equation}\label{equ:matroid}
	f(S)\geq \frac{1}{2}\left[f(S^*)-(k-k^{'}-|S_n|)\tau^+\right],
	\end{equation}
	where $S^*$ is the optimal set of problem (\ref{equ:matroidp}), $k^{'}=|A|=|C|$ such that
	\begin{equation}\label{equ:AC}
	\begin{array}{ll}
	A=\{v\in S\backslash S_n:\phi^{-1}(v)\in S^*\cap B_n\},\\
	C=\{v\in S^*\cap B_n:\phi(v)\in S\backslash S_n\}.
	\end{array}
	\end{equation}
	wherein $\phi:S^*\rightarrow S$ is the bijection whose existence has been guaranteed by the matroid property, and $\phi^{-1}:S\rightarrow S^*$ is its inverse.
\end{theorem}
\begin{proof}
	Let $k_n=|S^*\backslash S_n\cup B_n|$, we have
	\begin{equation}\label{equ:t23}
	\begin{array}{ll}
	&f(S^*)\\
	\leq& f(S\cup S^*)\leq f(S)+\sum\limits_{v\in S^*\backslash S}f(v|S)\\
	\leq& f(S)+\sum\limits_{v\in S^*\backslash (S_n\cup B_n)}f(v|S)+\sum\limits_{v\in S^*\cap B_n}f(v|S)\\
	\leq& f(S)+k_n\tau^-+\sum\limits_{v\in C}f(v|S)+\sum\limits_{v\in (S^*\cap B_n)\backslash C}f(v|S)\\
	\leq& [f(S)+\sum\limits_{v\in A}f(v|S)]+k_n\tau^-+(k-k_n-k^{'})\tau^+\\
	\leq& 2f(S)-\sum\limits_{v\in S\backslash (S_n\cup A)}f(v|S)-\\
	&\sum\limits_{v\in S_n}f(v|S)+k_n\tau^-+(k-k_n-k^{'})\tau^+\\
	\leq& 2f(S)-\sum\limits_{v\in S\backslash (S_n\cup A)}f(v|S)-\\
	&|S_n|\tau^++k_n\tau^-+(k-k_n-k^{'})\tau^+\\
	\leq& 2f(S)-|S_n|\tau^++k_n\tau^-+(k-k_n-k^{'})\tau^+\\
	=& 2f(S)-\left[(k^{'}+|S_n|-k)\tau^++k_n(\tau^+-\tau^-)\right]\\
	\leq& 2f(S)-\left[(k^{'}+|S_n|-k)\tau^+\right].
	\end{array}
	\end{equation}
	The fifth inequality in (\ref{equ:t23}) is due to $|(S^*\cap B_n)\backslash C|=k-k_n-k^{'}$, the sixth inequality uses the fact
	\begin{equation}
	f(S)=\sum\limits_{v\in A}f(v|S)+\sum\limits_{v\in S\backslash (S_n\cup A)}f(v|S)+\sum\limits_{v\in S_n}f(v|S),
	\end{equation}
	the seventh inequality in (\ref{equ:t23}) is a result of monotonicity, the eighth inequality in (\ref{equ:t23}) is due to $\tau^+\geq \tau^-$.
	
	By rearranging (\ref{equ:t23}), we have
	\begin{equation}
	f(S)\geq \frac{1}{2}\left[f(S^*)-(k-k^{'}-|S_n|)\tau^+\right].
	\end{equation}
	This completes the proof.
\end{proof}

\newpage

\begin{figure*}[htp]\vspace{-3mm}
	\begin{center}
		\includegraphics[width=1\linewidth]{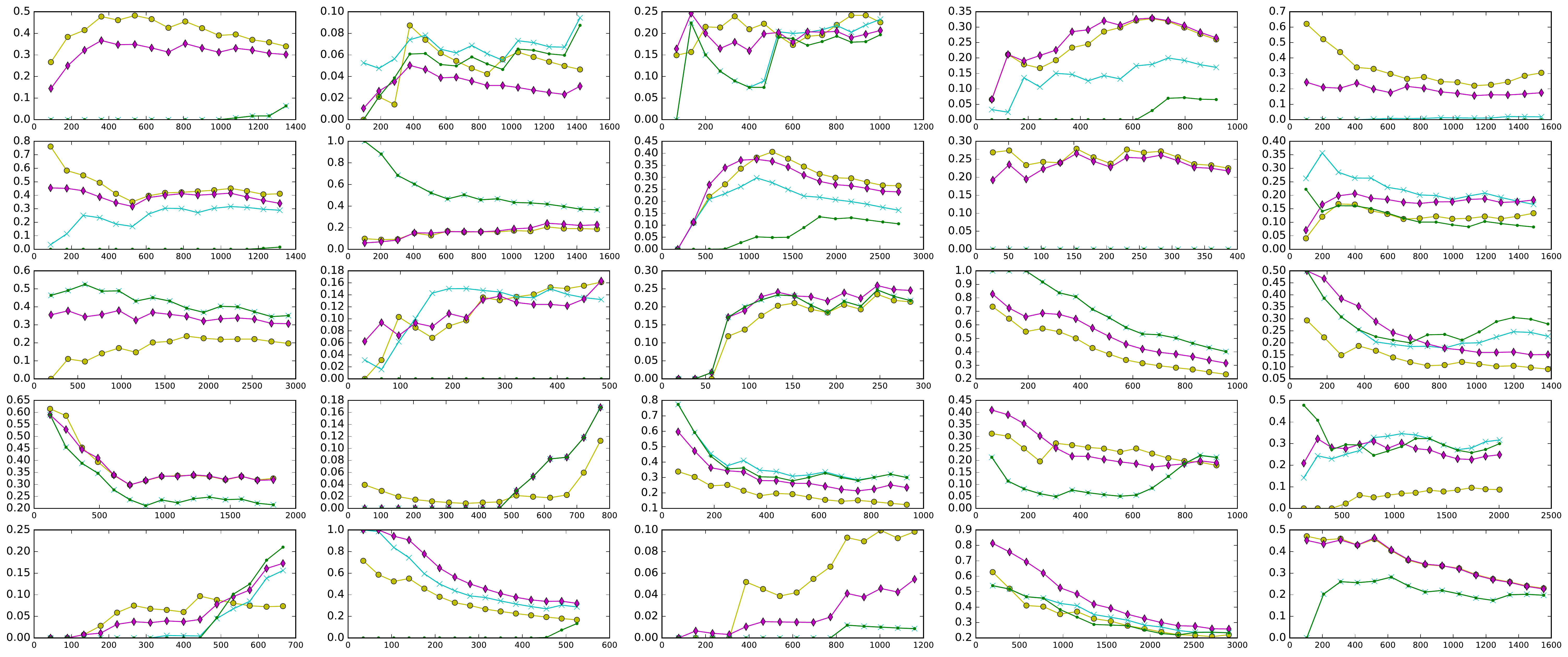}
	\end{center}\vspace{-5mm}
	\caption{Recall of the summaries generated by lazy greedy (``{\color{yellow}{$\bullet$}}''), sieve-streaming ( ``{\color{cyan}{$\times$}}''), stream clipper (``{\color{magenta}{$\blacklozenge$}}'') and the first $15\%$ frames (``{\color{green}{$\cdot$}}'') comparing to reference summaries of different sizes between $[0.02|V|, 0.32|V|]$ based on ground truth score (voting from $15$ users) on $25$ videos from SumMe. Each plot associates with a video.}
	\label{fig:videoR}
\end{figure*}\vspace{-3mm}

\begin{figure*}[htp]\vspace{-1mm}
	\begin{center}
		\includegraphics[width=1\linewidth]{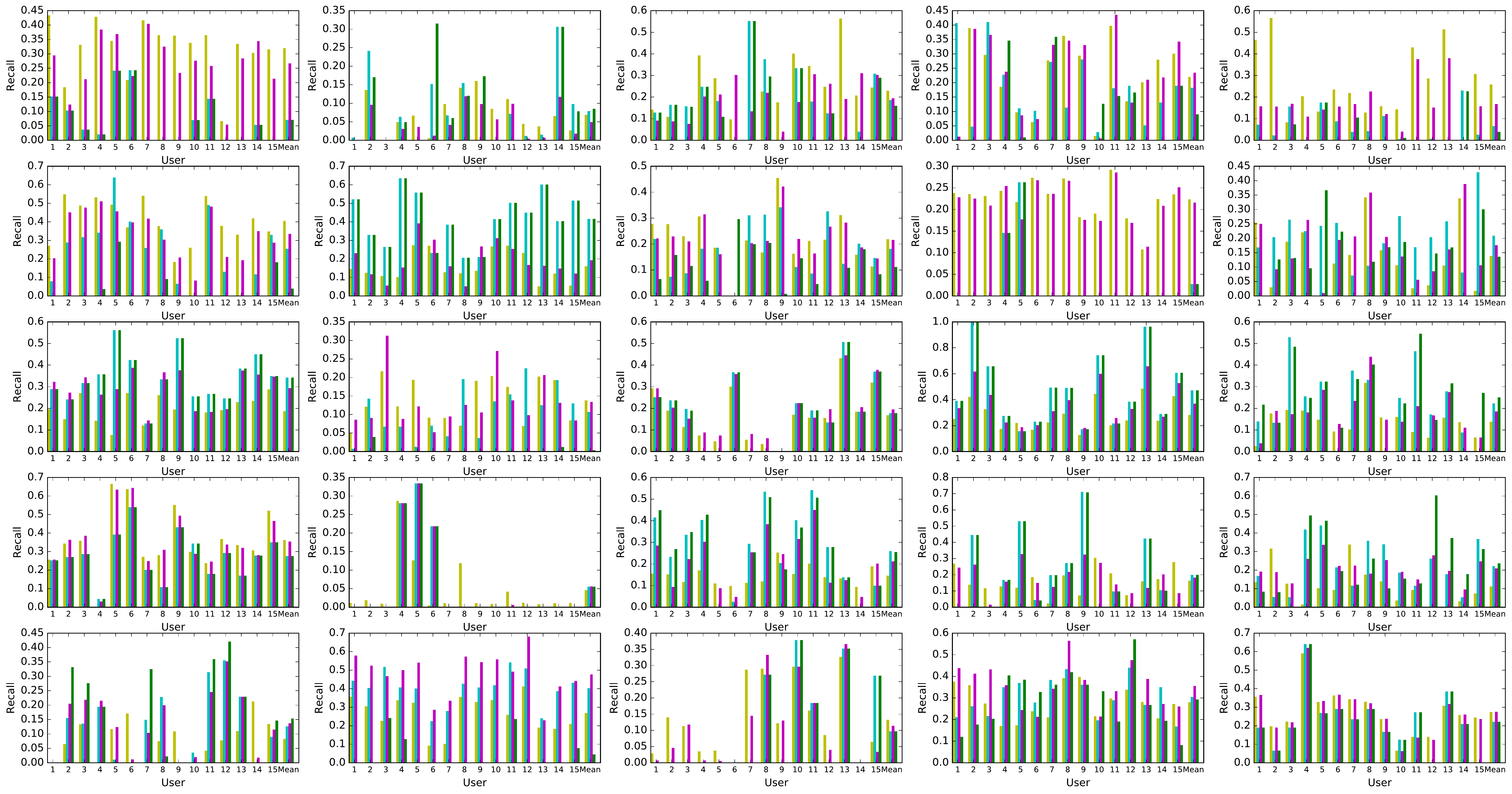}
	\end{center}\vspace{-2mm}
	\caption{Recall of the summaries generated by greedy (yellow bar), sieve-streaming ( cyan bar), stream clipper (magenta bar) and the first $15\%$ frames (green bar) comparing to reference summaries from $15$ users on $25$ videos from SumMe dataset. Each plot associates with a video.}
	\label{fig:videoR_all}
\end{figure*}\vspace{-1mm}

\vspace{-1mm}
\begin{table*}[htp]
	\caption{Information of SumMe dataset and time cost (CPU seconds) of different algorithms.}
	\begin{center}
		\begin{tabular}{l|c||c|c|c}
			\hline
			Videoname &\#frames &Lazy Greedy &Sieve-streaming &Stream Clipper\\
			\hline
			Air Force One &4494 &907.3712 &3.9182  &75.5463 \\
			Base jumping &4729 &164.1434 &5.5865  &30.2570 \\
			Bearpark climbing &3341 &177.8583 &3.7311  &34.2229 \\
			Bike polo &3064 &96.5305 &3.9578  &24.7901 \\
			Bus in rock tunnel &5131 &505.7766 &6.0088  &130.8577 \\
			Car over camera &4382 &146.9416 &5.3323  &59.2436 \\
			Car railcrossing &5075 &852.1686 &5.2265 &123.0835 \\
			Cockpit landing &9046 &669.8063 &12.3186 &103.6095 \\
			Cooking &1286 &30.0717 &1.2868 &4.9270\\
			Eiffel tower &4971 &304.2690 &5.4755 &81.4899\\
			Excavators river crossing &9721 &1507.3028 &13.8139 &283.7986 \\
			Fire Domino &1612 &34.2871 &1.8814 &9.4465 \\
			Jumps &950 &15.0508 &0.9055 &5.1711\\
			Kids playing in leaves &3187 &221.4644 &3.4660 &37.7304 \\
			Notre Dame &4608 &169.1235 &5.1406 &69.4589 \\
			Paintball &6096 &763.3255 &6.7853 &114.3241 \\
			Paluma jump &2574 &210.8670 &2.5342 &25.5281 \\
			Playing ball &3120 &132.7437 &3.2250 &14.0948\\
			Playing on water slide &3065 &111.7358 &3.4088 &30.9812 \\
			Saving dolphines &6683 &435.0732 &7.3322 &78.4372\\
			Scuba &2221 &45.6177 &2.5213 &8.3734 \\
			St Maarten Landing &1751 &19.0717 &2.8701 &3.5580 \\
			Statue of Liberty &3863 &160.7075 &4.0164 &56.4238 \\
			Uncut evening flight &9672 &718.7015 &14.6717 &122.7112 \\
			Valparaiso downhill &5178 &428.3941 &6.0002 &70.6994\\
			\hline
		\end{tabular}\label{table:SumMe}
	\end{center}
\end{table*}